\documentclass[11pt]{article}

\usepackage[margin=1in]{geometry}
\usepackage[T1]{fontenc}
\usepackage{lmodern}
\usepackage{microtype}
\usepackage{authblk}
\usepackage[numbers,sort&compress]{natbib}

\usepackage{amsmath,amssymb,amsfonts,mathtools,amsthm}
\usepackage{algorithm}
\usepackage{algorithmic}
\usepackage{graphicx}
\usepackage{textcomp}
\usepackage{xcolor}

\usepackage{booktabs}
\usepackage{adjustbox}
\usepackage{tikz}
\usetikzlibrary{arrows.meta,positioning,calc}
\usepackage[colorlinks=true,linkcolor=blue!60!black,citecolor=blue!60!black,urlcolor=blue!60!black]{hyperref}

\newtheorem{theorem}{Theorem}
\newtheorem{proposition}{Proposition}
\newtheorem{lemma}{Lemma}
\newtheorem{definition}{Definition}
\newtheorem{corollary}{Corollary}
\newtheorem{assumption}{Assumption}
\theoremstyle{remark}

\definecolor{myblue}{RGB}{35,92,255}
\definecolor{mypurple}{RGB}{130,45,220}
\definecolor{blockfill}{RGB}{229,239,255}

\newcommand{\A}{\mathcal{A}}
\newcommand{\U}{\mathcal{U}}

\newcommand{\E}{\mathbb{E}}

\newcommand{\norm}[1]{\left\lVert #1 \right\rVert}
\newcommand{\eps}{\varepsilon}

\definecolor{colBlue}{RGB}{24, 95, 165}
\definecolor{colRed}{RGB}{216, 90, 48}
\definecolor{colGray}{RGB}{95, 94, 90}
\definecolor{colGrayLight}{RGB}{180,178,169}
 
\tikzset{
  arr/.style      = {-{Latex[length=5pt,width=4pt]}, line width=1.6pt},
  arrRR/.style    = {-{Latex[length=4pt,width=3.5pt]}, line width=1.1pt,
                     dashed, dash pattern=on 4pt off 2.5pt},
  ck/.style       = {circle, fill=white, draw=colGrayLight,
                     line width=0.6pt, inner sep=0pt, minimum size=5pt},
  startdot/.style = {circle, fill=colGray, inner sep=0pt, minimum size=4pt},
  lbl/.style      = {font=\small},
  lblrr/.style    = {font=\small, text=colGray},
  ckl/.style      = {font=\scriptsize, text=colGrayLight},
}
\title{Counterfactual Certification for Sequential Unlearning\\ with Long-Run Performance Guarantees}

\author[1]{Hanna Benarroch\thanks{Corresponding author: \texttt{hanna.benarroch@ens.psl.eu}}}
\author[2]{Jamal Atif}
\author[1]{Olivier Capp\'e}

\affil[1]{DI ENS, \'Ecole normale sup\'erieure, Universit\'e PSL, CNRS, 75005 Paris, France}
\affil[2]{CMAP, \'Ecole polytechnique, Institut Polytechnique de Paris, 91120 Palaiseau, France}

\date{}

\begin{document}

\maketitle

\begin{abstract}
Certified machine unlearning ensures the proximity between the
unlearned model and a model retrained without the data to be forgotten.
Extending this framework to sequential deletion requests, typically arriving as groups of variable size (e.g., all records of one user or of coordinated users), is challenging. For unlearning methods based on fine-tuning on the retain sets, this leads to either degrading performance or increasing computation as the number of unlearning requests grows.
For independent successive requests, we instead certify each model release against a counterfactual unlearning
trajectory in which the group being deleted never took part. With a fixed fine-tuning budget per deletion, we show that, for smooth strongly convex objectives, this approach keeps the performance loss bounded, uniformly in the number of requests.
Experiments with multiclass logistic regression on FEMNIST show that the resulting algorithm needs fewer corrective steps than existing sequential methods and outperforms training-time differential privacy, particularly when group sizes are heterogeneous.
\end{abstract}

\section{Introduction}
\label{sec:introduction}

Training data may need to be removed after a model has been trained, whether to honor 
deletion requests, comply with privacy laws, or meet broader regulatory requirements 
on deployed AI systems \cite{marino2025position}. Retraining from scratch without the data is 
the gold standard but is often computationally prohibitive; machine unlearning 
seeks to approximate its effect at a substantially lower computational cost, either by enforcing 
exact equivalence with retraining \cite{cao2015towards} or, for randomized algorithms, 
by matching the output distribution of the unlearned model to that of retraining on 
the retained data \cite{ginart2019,sekhari2021}.
We work in the latter setting, and specifically on \emph{certified} unlearning, where 
the deletion procedure comes with an explicit, provable guarantee.
The sequential unlearning setting is challenging as requests arrive asynchronously over
an indefinite horizon, and each usually concerns a \emph{group} of records
(all the records of a user exercising their right to erasure, of
an organization, or of a data source), whose size and content are only revealed upon
arrival of the request.

For a single request, the natural reference for unlearning is a (fictitious) rerun
of the learning algorithm without the deleted data point. 
In contrast, for sequential unlearning the model released at request $r$
results from a history of learning and unlearning operations, and one must first ask:
\emph{against which reference should it be certified?} Following Ginart
et al.~\cite{ginart2019}, most certified sequential analyses
\cite{chien2025langevinunlearningnewperspective,neel2020descenttodeletegradientbasedmethodsmachine,
guo2020certified, mu2025rewind, zhang2024towards, qiao2025hessianfree} use the model
retrained from scratch with all deletions applied, i.e., they treat the successive deleted groups
$G_1,\ldots,G_r$ as a single group $G_1\cup\cdots\cup G_r$.
This joint treatment is natural for dependent requests, such as those of the same individual
or of coordinated individuals, but conservative for independent ones. It is also costly, as
the reference at round $r$ is $\sum_s|G_s|$ records away from the training set, so that the amount of injected noise and/or
the computational effort necessarily grow with $r$
\cite{chien2025langevinunlearningnewperspective,neel2020descenttodeletegradientbasedmethodsmachine,
guo2020certified}
(see also Appendix~\ref{app:baseline-guarantees}).
We call this reference the \emph{cumulative-deletion reference} and propose instead a
\emph{counterfactual sequential reference}: at request $r$, we compare the release with the
trajectory the algorithm \emph{would have produced had $G_r$ never been in the training
set}, the previous deletions $G_1,\ldots,G_{r-1}$ being applied exactly as in the actual
run.

To illustrate the need for such a specific framework for sequential unlearning, consider as an
example the use of a training-time differential privacy (DP) mechanism. For a single
unlearning request, it is well established that DP offers an unlearning guarantee ``for free''~\cite{guo2020certified,sekhari2021}. Note that
DP does not necessarily constitute a viable option for unlearning as it may cause performance
degradation even in the absence of any actual unlearning request and it responds passively to unlearning requests,
which might simply be unacceptable when those requests are dictated by agency principles. However, when the number
of deletion requests exceeds one, the guarantees that can be
provided to the users under the cumulative-deletion reference can only decrease (due to the group privacy
effect). The proposed counterfactual framework is introduced to eliminate this counterintuitive first-mover
advantage, which is artificial for independent requests, by simply requiring that successive requests
be treated similarly to the first one. As we will see, with this proposed definition, training-time DP can provide
non-vanishing unlearning guarantees for independent deletion requests, by paying a cost that depends only
on the maximal number of data points that need to be deleted per request.

However, training-time DP falls short of being adaptive to the successive unlearning requests,
and we focus instead on the use of noisy gradient fine-tuning on the retained data. When
using the proposed unlearning reference together with a fixed budget of
fine-tuning steps per deletion request, we obtain a method that we refer to as
\emph{Counterfactual Sequential Unlearning} (CoSU). In the smooth strongly convex
setting, adapting the shifted-interpolation framework of Bok et
al.~\cite{bok2024shiftedinterpolationdifferentialprivacy}, we derive 
per-request noise scales that depend on the size of the current request and remain
\emph{bounded uniformly} in $r$. In addition, we provide an excess-risk
bound that reaches a \emph{plateau}, suggesting that this approach can successfully handle
long sequences of deletion requests.

On FEMNIST with user-level deletions, we measure the privacy--utility
trade-off of CoSU, provide empirical evidence with membership inference attacks that
independent requests are protected as certified while dependent requests certified
separately may not be, and find that CoSU outperforms sequential baselines (including 
training-time DP), in particular for heterogeneous group sizes.

\section{Related Work}
\label{sec:related-work}

\paragraph{Sequential certified unlearning}
Ginart et al.~\cite{ginart2019} introduced the cumulative-deletion reference, which most
subsequent certified sequential methods retain
\cite{chien2025langevinunlearningnewperspective,
neel2020descenttodeletegradientbasedmethodsmachine,
guo2020certified,mu2025rewind,zhang2024towards,qiao2025hessianfree}.
When requests may depend on previous releases, the sequential definition is no longer
well posed \cite{gupta2021}, and the protection of a deleted record can depend on that of
retained ones \cite{chourasia2023forgetunlearning}; we therefore consider non-adaptive
requests (Assumption~\ref{ass:adaptive}). Neel et
al.~\cite{neel2020descenttodeletegradientbasedmethodsmachine} also distinguish methods
whose entire persistent state is certified from methods that keep a hidden clean state,
which is more efficient but may leak deleted records across later updates
\cite{chourasia2023forgetunlearning,cohen2026protecting}; we adopt the former
(Assumption~\ref{ass:stateless}).

\paragraph{Sequential methods (using the cumulative-deletion reference)}
Descent-to-Delete (D2D) alternates corrective gradient steps and Gaussian perturbation;
when no hidden state is kept, its number of iterations grows with the number of requests
\cite{neel2020descenttodeletegradientbasedmethodsmachine}. Langevin Unlearning uses noisy
gradient descent for both learning and unlearning, and recursively determines the number
of unlearning steps of each deletion through a R\'enyi-divergence bound
\cite{chien2025langevinunlearningnewperspective}. Guo et al.~\cite{guo2020certified}
combine objective perturbation at training time with a Newton-type removal step,
maintaining a certificate on the accumulated gradient residual and retraining once it is
exhausted. Zhang et al.~\cite{zhang2024towards} extend Newton-style certified unlearning
to deep models and sequential removals, but perturb only the final released model, without
the per-request guarantee considered here. Qiao et al.~\cite{qiao2025hessianfree} avoid
explicit Hessian inversion by storing per-sample influence-like quantities; their method
relies on a clean parameter and clean data-dependent statistics kept across requests, and
therefore falls outside our stateless setting.

\paragraph{Counterfactual references}
Several works define deletion guarantees relative to a counterfactual execution in
which the deleted data was never present. Garg et al.~\cite{garg2020formalizing} state
this principle for a general data collector: the collector's state must be the same
whether a party submits data and later requests deletion, or never communicates with
the collector at all. Our definition enforces this comparison for independent
requests (Proposition~\ref{prop:semantics}, Appendix~\ref{app:proofs-definition}). In
machine unlearning, Sekhari et al.~\cite{sekhari2021} follow this principle for a single
deletion request, with a reference obtained by training without the deleted records, then
running the unlearning procedure with an empty request; we extend this reference to sequences of requests.
Chourasia and Shah~\cite{chourasia2023forgetunlearning} also use a counterfactual
execution for non-adaptive requests, but accept any reference that does not depend on the
deleted record, including the cumulative-deletion reference $\A(D_r)$; we fix the reference
to the counterfactual trajectory.

\paragraph{Privacy amplification by iteration}
Our noise calibration relies on privacy amplification by iteration
(PABI)~\cite{feldman2018amplification}: when only the final iterate is released,
contractive noisy iterations leak much less than naive composition suggests
\cite{chourasia2021dpdynamics,ye2022differentially,altschuler2022noisysgd}.
Our main technical tool is the shifted-interpolation of Bok
et al.~\cite{bok2024shiftedinterpolationdifferentialprivacy}, formulated in the $f$-DP
framework, which we apply to pairs of actual and counterfactual unlearning
trajectories. Our technical contribution (Theorem~\ref{thm:eps-delta-tracking}) is to
allocate the effect of the noise across deletion requests, so that each pair is certified
with a request-dependent noise scale $\sigma_r$.

\section{Counterfactual Sequential Unlearning}
\label{sec:definition}

\subsection{Notation}
\label{sec:process}

A \emph{record} (or data point) is an element of a data domain $\mathcal D$. A dataset of
size $n$ is a collection of $n$ records,
\[
D=(d_1,\ldots,d_n)\in\mathcal D^n.
\]
We write $\mathcal D^* := \bigcup_{n \ge 0} \mathcal D^n$ for the set of all datasets
of arbitrary finite size. A \emph{group}
$G\subseteq D$ is a set of records of $D$, of size $|G|$, and $D\setminus G$ is the dataset
of the remaining records. Two groups are disjoint if they share no record, and the
\emph{content} of a group is the values of its records.

Models are parameters $\theta\in\Theta\subseteq\mathbb R^p$, and $\mathcal P(\Theta)$ denotes
the probability measures on $\Theta$. A randomized \emph{learning algorithm} is a Markov kernel
\[
\A:\mathcal D^*\to\mathcal P(\Theta),
\]
and we write $\theta\sim\A(D)$ for a draw from $\A(D)$. A \emph{sequential unlearning
procedure} is a sequence of Markov kernels
\[
\U_r:\Theta\times\mathcal D^*\to\mathcal P(\Theta),\qquad r\ge1,
\]
which map the current parameter $\theta$ and the retained dataset $D'$ to a random updated parameter
$\theta^+\sim\U_r(\theta,D')$.

\begin{definition}[$(\varepsilon,\delta)$-indistinguishability \cite{neel2020descenttodeletegradientbasedmethodsmachine}]
\label{def:indistinguishability}
Two $\Theta$-valued random variables $X$ and $Y$ satisfy $X\simeq_{\varepsilon,\delta}Y$ if,
for every measurable set $S\subseteq\Theta$,
\[
e^{-\varepsilon}\bigl(\Pr(Y\in S)-\delta\bigr)\le\Pr(X\in S)\le e^{\varepsilon}\Pr(Y\in S)+\delta .
\]
\end{definition}

\subsection{Sequential unlearning process}

Starting from the original dataset $D_0:=D$, the model owner receives an online sequence
of deletion requests $G_1,G_2,\ldots$, each of which removes a group of records from the
current dataset:
\[
D_0\xrightarrow{\;G_1\;}D_1\xrightarrow{\;G_2\;}D_2\xrightarrow{\;G_3\;}\cdots,
\qquad D_r:=D_{r-1}\setminus G_r .
\]
The groups $G_r$ are pairwise disjoint. A group may be a single record, all the records of a
user (as in \emph{user-level deletion}, our running example), or all the records of an
organization. Request $r$ is certified at a privacy level $(\varepsilon_r,\delta_r)$, which is
known, together with the size $g_r:=|G_r|$, only when the request arrives. The number of
requests processed so far and the size and privacy level of each
request are public; which records a group contains is private.

The initial model and the \emph{actual trajectory} are defined by
\[
\theta_0\sim\A(D_0),\qquad
\theta_r\mid\theta_{r-1}\sim\U_r(\theta_{r-1},D_r),\quad r\ge1.
\]
We make the following assumptions. Assumption~\ref{ass:independent-requests} is specific to
the counterfactual reference and is discussed in Section~\ref{sec:counterfactual-reference}.

\begin{assumption}[Public maps]
\label{ass:independence}
$\A$ and $(\U_r)_{r\ge1}$ are fixed independently of the data. $\U_r$
may depend on the requests only through public quantities: the counter $r$ and the sizes and
privacy levels $(g_s,\varepsilon_s,\delta_s)_{s\le r}$ of the requests received so far.
\end{assumption}

\begin{assumption}[Non-adaptive requests]
\label{ass:adaptive}
The requests (groups, sizes and privacy levels) are independent of the released models
$\theta_0,\theta_1,\ldots,\theta_r$.
\end{assumption}

\begin{assumption}[Stateless model]
\label{ass:stateless}
Between requests, the owner keeps only $\theta_r$, $D_r$ and the public quantities of
Assumption~\ref{ass:independence}. No hidden state involving deleted records, such as
cached gradients or checkpoints, is kept.
\end{assumption}

\begin{assumption}[Single release]
\label{ass:single-release}
When $\theta_r$ is published, the previous release $\theta_{r-1}$ is withdrawn: it can no
longer be served, queried or otherwise accessed.\footnote{Otherwise, comparing successive releases enables
differencing attacks that can reconstruct deleted data \cite{cohen2026protecting}.}
\end{assumption}

\begin{assumption}[Independent requests]
\label{ass:independent-requests}
For every $r$, given the size and privacy level of request $r$, the law of the other requests is
the same whether or not $G_r$ was ever in the training set, and does not depend on the
content of $G_r$.
\end{assumption}

\subsection{Cumulative-deletion reference}
The standard reference, introduced by Ginart et al.~\cite{ginart2019}, is a fresh run of
the learning algorithm on the retained dataset $D_r=D\setminus(G_1\cup\cdots\cup G_r)$:
\[
\theta_r^{\mathrm{cd}}\sim\A(D_r).
\]

\begin{definition}[Cumulative-deletion sequential unlearning \cite{ginart2019}]
\label{def:naive_sequential_unlearning}
A pipeline $(\A,(\U_s)_{s\ge1})$ satisfies $(\varepsilon_r,\delta_r)$-cumulative-deletion
sequential unlearning at request $r$ if, for every dataset $D$ and every request sequence
$(G_1,\ldots,G_r)$,
\[
\theta_r\simeq_{\varepsilon_r,\delta_r}\theta_r^{\mathrm{cd}},
\]
where $\theta_r$ and $\theta_r^{\mathrm{cd}}$ are defined as above.
\end{definition}

\subsection{Counterfactual sequential reference}
\label{sec:counterfactual-reference}

With the \emph{counterfactual} reference, a group that has already been unlearned is not
counted again: request $r$ only has to certify the deletion of the current group $G_r$,
and not that of the previous groups $G_1,\ldots,G_{r-1}$. The
reference is therefore the trajectory obtained by removing $G_r$ from the initial
dataset and processing the previous deletions $G_1,\ldots,G_{r-1}$ exactly as in the
actual run (Figure~\ref{fig:standard-vs-counterfactual-reference},
Appendix~\ref{app:illustration}). For independent requests, it describes the world in
which $G_r$ was never collected, with request $r$ still issued as a void request of the
same size and privacy level. This is the comparison required by the deletion-compliance notion
of Garg et al.~\cite{garg2020formalizing} (Proposition~\ref{prop:semantics},
Appendix~\ref{app:proofs-definition}).

This comparison keeps the other requests fixed, which is legitimate only if they would
have been issued even without $G_r$. This holds when each request depends only on its
own group and on independent randomness (Assumption~\ref{ass:independent-requests}),
for instance when users leave independently of each other. It fails when one request
triggers another, for instance when the users of an organization leave together:
without $G_r$, the requests it triggers would not have been issued either, so a
reference that keeps them no longer describes the world without $G_r$. Such dependent
requests must be merged into a single group (Proposition~\ref{prop:necessity}) or, when
they are issued at different times, certified jointly (Appendix~\ref{app:dependent}).

For each $s\ge 0$, we define the \emph{counterfactual retained
dataset}, which is the dataset retained after the first $s$ requests, with $G_r$
removed in advance:
\[
D_s^{\langle r\rangle}
:=
\begin{cases}
D_s\setminus G_{r}, & s<r,\\
D_s, & s\ge r.
\end{cases}
\]
In particular, $D_{r-1}^{\langle r\rangle}=D_r$. 
The \emph{counterfactual
reference trajectory} runs the actual pipeline on these datasets:
\[
\theta_0^{\langle r\rangle}\sim\A\big(D_0^{\langle r\rangle}\big),
\qquad
\theta_s^{\langle r\rangle}\mid\theta_{s-1}^{\langle r\rangle}
\sim\U_s\big(\theta_{s-1}^{\langle r\rangle},D_s^{\langle r\rangle}\big),
\quad s\ge 1.
\]
From episode $r$ on, since $D_s^{\langle r\rangle}=D_s$, it applies the same updates
$\U_s(\cdot,D_s)$ as the actual trajectory.

\begin{definition}[Counterfactual sequential unlearning]
\label{def:counterfactual_sequential_unlearning}
Let the requests satisfy Assumption~\ref{ass:independent-requests}. A pipeline
$(\A,(\U_s)_{s\ge1})$ satisfies $(\varepsilon_r,\delta_r)$-counterfactual sequential
unlearning at request $r$ if, for every dataset $D$ and every request sequence
$(G_1,\ldots,G_r)$,
\[
\theta_r\simeq_{\varepsilon_r,\delta_r}\theta_r^{\langle r\rangle}.
\]
\end{definition}

Definitions~\ref{def:naive_sequential_unlearning}
and~\ref{def:counterfactual_sequential_unlearning} are not directly comparable, since
they certify the same release against different references. After request $r$, the
cumulative-deletion definition certifies the whole deleted set
$G_1\cup\dots\cup G_r$ at once, whereas the counterfactual definition gives a 
certificate for the current request $G_r$ only.

\subsection{Properties of counterfactual sequential unlearning}
\label{sec:consequences-counterfactual}

The counterfactual reference raises two questions that do not arise with the
cumulative-deletion reference. First, a reference is valid only if its law does not depend on the content of the group
being deleted. This is
immediate for the cumulative-deletion reference $\theta_r^{\mathrm{cd}}\sim\A(D_r)$. For the
counterfactual reference, it must be checked, because its trajectory goes through the
kernels $\U_1,\ldots,\U_{r-1}$, which were applied in the actual process while $G_r$ was
still in the training set.

\begin{proposition}[Independence from the deleted group]
\label{prop:independence-reference}
Under Assumptions~\ref{ass:independence} and~\ref{ass:stateless}, the law of
$\theta_r^{\langle r\rangle}$ depends on $D$ only through $D\setminus G_r$, the groups
$G_1,\ldots,G_{r-1}$, which are contained in $D\setminus G_r$, and the sizes and privacy levels of
the requests. In particular, it is unchanged if the content of $G_r$ is replaced by
arbitrary values in $\mathcal D$.
\end{proposition}

Second, the successive retained datasets of the cumulative-deletion reference are nested,
$D_0\supset D_1\supset\cdots\supset D_r$. In our approach, however, each request defines its own
counterfactual trajectory: for $q\neq r$ and $s<\min(q,r)$,
$D_s^{\langle q\rangle}=D_s\setminus G_q$ and $D_s^{\langle r\rangle}=D_s\setminus G_r$ differ. 
The question is then whether the
certificate of an earlier request remains valid after later requests are processed.

\begin{proposition}[Persistence]
\label{prop:persistence-postprocessing}
Let Assumptions~\ref{ass:independence}--\ref{ass:stateless} hold, and let
$\theta_s^{\langle r\rangle}$, $s\ge r$, continue the counterfactual trajectory of request $r$.
If $\theta_r\simeq_{\varepsilon_r,\delta_r}\theta_r^{\langle r\rangle}$, then
$\theta_s\simeq_{\varepsilon_r,\delta_r}\theta_s^{\langle r\rangle}$ for every $s>r$.
\end{proposition}

After $r$ requests, the current release thus carries one certificate per deleted group. These
certificates do not, however, combine into a single guarantee with respect to the
cumulative-deletion reference $\A(D_r)$. The proofs of Propositions~\ref{prop:independence-reference} and~\ref{prop:persistence-postprocessing} are
given in Appendix~\ref{app:proofs-definition}.

\subsection{Training-time privacy as a baseline}
\label{sec:no-provisioning}

A natural baseline is to train with differential privacy (DP) and perform no update at
deletion time, so that $\theta_r\sim\A(D)$ for every $r$. The counterfactual reference is
then $\theta_r^{\langle r\rangle}\sim \A(D\setminus G_r)$, and
Definition~\ref{def:counterfactual_sequential_unlearning} reduces to DP with respect to
the removal of $G_r$. A learner that is $(\varepsilon,\delta)$-DP for groups of at most
$g_{\max}$ records therefore satisfies this definition for every request of size at most
$g_{\max}$ and privacy level at least $(\varepsilon,\delta)$
(Proposition~\ref{prop:no-provisioning}, Appendix~\ref{app:proofs-definition}).

However, in addition to the limitations of differential privacy for unlearning listed 
in the introduction, the noise scale must be
calibrated before training for the largest group size and the strictest privacy level.
The resulting loss of accuracy is measured in Section~\ref{sec:exp-dp}.

\section{The CoSU Algorithm}
\label{sec:method}

We now instantiate the definition with a concrete learning and
unlearning scheme, which we call Counterfactual Sequential Unlearning (CoSU).
The training objective ($r=0$, with $D_0=D$) and the retained objective after $r$
deletion requests are
\[
f_r(\theta):=\sum_{d\in D_r}\ell(\theta,d)+\Omega_\lambda(\theta),\qquad r\ge0,
\]
where \(\Omega_\lambda\) is a fixed regularization term.

\begin{assumption}[Smooth strongly convex objectives]
\label{ass:contraction}
Each loss $\ell(\cdot,d)$ is convex, differentiable and $M$-Lipschitz, and, for every subset
$D'\subseteq D$, $\sum_{d\in D'}\ell(\cdot,d)+\Omega_\lambda$ is $m$-strongly convex and
$L$-smooth, with public constants $M,m,L$. With step size $\eta=1/L$, every gradient map
$\theta\mapsto\theta-\eta\nabla f_r(\theta)$ is $c$-contractive, $c:=1-\eta m<1$.
\end{assumption}

Starting from a data-independent $\theta_{\mathrm{init}}$, \(\A\) runs \(T\) steps of noisy gradient descent on
\(f_0\), with noise scale \(\sigma_0\geq 0\). For $k=0,\dots,T-1,$
\[
    \theta_{k+1}
    =
    \theta_k - \eta \nabla f_0(\theta_k)
    +
    \sqrt{2\eta}\,\sigma_0 Z_{k+1},
\]
where \(Z_{k+1}\sim \mathcal N(0,I)\) are independent standard Gaussian vectors.
The case \(\sigma_0=0\) corresponds to a deterministic training phase.

After the \(r\)-th deletion request, the unlearning algorithm \(\U_r\) performs
\(K\ll T\) corrective noisy gradient steps on the updated objective \(f_r\), with
noise scale \(\sigma_r\ge0\). Writing
\[
    N_r := T+rK,
\]
the \(r\)-th unlearning episode consists in, for $k=N_{r-1},\dots,N_r-1,$
\[
    \theta_{k+1}
    =
    \theta_k - \eta \nabla f_r(\theta_k)
    +
    \sqrt{2\eta}\,\sigma_r Z_{k+1}.
\]
In response to the $r$-th unlearning request, the algorithm releases $\theta_{N_r}$. Each
request requires $K$ gradient steps, and using $K\ll T$ keeps
this cost far below that of retraining. The noise scale $\sigma_r$ is calibrated when
request $r$ arrives (Algorithm~\ref{alg:cosu}), so that $\theta_{N_r}$ satisfies
counterfactual sequential unlearning at level $(\varepsilon_r,\delta_r)$. With $\sigma_0=0$, the
model released before any request is the non-private one; a
positive $\sigma_0$ lowers utility after training but makes the first 
unlearning request easier to satisfy (Proposition~\ref{prop:training-noise}).

\section{Theoretical Guarantees}
\label{sec:theory}

In this section, we state the unlearning and utility guarantees for CoSU
(detailed in Algorithm~\ref{alg:cosu}).
Each request is certified within the Gaussian DP framework~\cite{dong2022gaussian} in which a
guarantee at level
$\bar\mu_r:=\sqrt{2\log(1/\delta_r)+2\varepsilon_r}-\sqrt{2\log(1/\delta_r)}$ implies
$(\varepsilon_r,\delta_r)$-indistinguishability (Lemma~\ref{lem:gdp-to-dp}).

\begin{algorithm}[t]
\caption{CoSU with online calibration (Theorem~\ref{thm:eps-delta-tracking})}
\label{alg:cosu}
\begin{algorithmic}[1]
\STATE Train: $T$ noisy gradient steps on $f_0$ with noise scale $\sigma_0\ge0$; release $\theta_{N_0}$;
$W_0\gets2\eta\sigma_0^2S_2(T)$
\FOR{each request $r$, with public size $g_r$ and level $\bar\mu_r$}
\STATE $b_r\gets g_rc^KI_{r-1}$;\quad $P_r\gets c^{2K}W_{r-1}$
\STATE $\bar x_r\gets P_r\min_{s<r,\ W_s>0}\,g_rI_s/(c^{K(r-s)}W_s)$
\STATE $x_r\gets\min\{\bar\mu_r^2P_r/b_r,\ \bar x_r\}$ \ ($x_r\gets0$ if $P_r=0$)
\STATE $\sigma_r^2\gets\dfrac{(b_r-x_r)^2}{2\eta S_2(K)\,(\bar\mu_r^2-x_r^2/P_r)}$
\ ($\sigma_r\gets0$ if $x_r=b_r$)
\STATE $K$ noisy gradient steps on $f_r$ with noise scale $\sigma_r$; release $\theta_{N_r}$
\STATE $W_r\gets P_r+2\eta\sigma_r^2S_2(K)$
\ENDFOR
\end{algorithmic}
\end{algorithm}

Let $S_2(k):=\sum_{u=0}^{k-1}c^{2u}$ and $W_0:=2\eta\sigma_0^2S_2(T)$. For each
request $r$, let
\begin{align*}
I_r&:=\eta M\,\frac{1-c^{N_r}}{1-c}, & b_r&:=g_r\,c^K I_{r-1},\\
P_r&:=c^{2K}W_{r-1}, & W_r&:=P_r+2\eta\sigma_r^2S_2(K).
\end{align*}
The quantity $b_r$ bounds the gap between the actual and counterfactual trajectories at
the release of request $r$. The quantity $W_r$ is the variance of the cumulative noise present in
the model after episode $r$: its first component, $P_r$, comes from past noise, injected
before episode $r$, and its second component represents the contribution of the noise added during the $r$-th
unlearning episode.

\begin{theorem}[Noise calibration for CoSU]
\label{thm:eps-delta-tracking}
Let Assumptions~\ref{ass:independence}--\ref{ass:contraction} hold. CoSU satisfies
$(\varepsilon_r,\delta_r)$-counterfactual sequential unlearning at every request
$r\ge1$ if each unlearning episode $r$ is run with the noise variance
\[
\sigma_r^2:=\frac{(b_r-x_r)^2}{2\eta S_2(K)\,\bigl(\bar\mu_r^2-x_r^2/P_r\bigr)},
\]
where $x_r:=\min\bigl\{\bar\mu_r^2P_r/b_r,\ \bar x_r\bigr\}$ and $\bar x_r$ is the value of
the minimization problem
\[
\bar x_r:=P_r\min_{0\le s<r,\ W_s>0}\frac{g_rI_s}{c^{K(r-s)}W_s},
\]
which is well defined when $P_r>0$. By convention, $x_r:=0$ and $x_r^2/P_r:=0$ when $P_r=0$, and $\sigma_r:=0$ when
$x_r=b_r$.
\end{theorem}

The quantity $x_r$ is the part of the gap $b_r$ covered by the past noise, so
that the noise of episode $r$ only has to cover the remaining part $b_r-x_r$. Although
$\sigma_r^2$ depends on the previous requests through $P_r$ and $x_r$, it remains bounded
uniformly in $r$.

\begin{corollary}[Uniformly bounded noise]
\label{cor:pay-as-you-go}
Let $h_r:=g_r/\bar\mu_r$ and
$\Gamma:=\frac{\eta M^2c^{2K}}{2}\frac{1+c}{1-c}$.
\begin{enumerate}
\item[(a)] For every request sequence and every $r$,
\[
\sigma_r^2\le\frac{h_r^2\,\Gamma}{1-c^{2K}}.
\]
\item[(b)] If $\sigma_0=0$ and all requests have the same size $g$ and level $\bar\mu$,
then, with $h:=g/\bar\mu$,
\begin{align*}
\sigma_1^2&=\frac{h^2\Gamma\,(1-c^T)^2}{1-c^{2K}},\\
\sigma_r^2&=h^2\Gamma\Bigl(1-\frac{2c^{T+(r-1)K}}{1+c^K}\Bigr)\quad\text{for } r\ge2,
\end{align*}
which increases to $h^2\Gamma$ as $r\to\infty$.
\end{enumerate}
\end{corollary}

By (a), $\sigma_r^2$ is bounded in terms of the size and level of request $r$ only:
the bound does not depend on the previous requests. It also decreases geometrically with
the number $K$ of corrective steps, since $\Gamma$ is proportional to $c^{2K}$. By (b),
for identical requests, the first request needs the largest noise,
$\sigma_1^2\approx h^2\Gamma/(1-c^{2K})$, to compensate for the absence of noise during training. The following
requests need less, $\sigma_r^2\le h^2\Gamma$, since the past noise has already covered part of the gap.

\begin{proposition}[Training noise]
\label{prop:training-noise}
With training noise $\sigma_0>0$, the first request of CoSU requires
\[
\sigma_1^2=c^{2K}\frac{S_2(T)}{S_2(K)}
\bigl(\sigma_0^\star(g_1,\bar\mu_1)^2-\sigma_0^2\bigr)_+ ,
\]
where
\[
\sigma_0^\star(g,\bar\mu):=\frac{g\,I_0}{\bar\mu\sqrt{2\eta S_2(T)}}.
\]
\end{proposition}

In particular, the first request requires no additional noise when
$\sigma_0\ge\sigma_0^\star(g_1,\bar\mu_1)$. Theorem~\ref{thm:eps-delta-tracking},
Corollary~\ref{cor:pay-as-you-go} and Proposition~\ref{prop:training-noise} are proved in
Appendix~\ref{sec:appendix-proof-CoSU}.

We now turn to the utility of the released models.

\begin{theorem}[Utility guarantee for CoSU]
\label{thm:final-excess-risk}
Let $\mathcal E_r:=\mathbb E\bigl[f_r(\theta_{N_r})-f_r(\theta_r^\star)\bigr]$ be the excess
empirical risk after request $r$, where $\theta_r^\star$ minimizes $f_r$, and let
$B:=LM^2/(2m^2)$. Under the calibration of Theorem~\ref{thm:eps-delta-tracking}, for every
$r$,
\begin{equation}
    \label{eq:findbetterlabel}
\mathcal E_r
\le
c^{Kr}\mathcal E_0
+
B\,c^{2K}\sum_{s=1}^{r}c^{K(r-s)}\Bigl(\frac{g_s^2}{1-c^K}+p\,h_s^2\Bigr),
\end{equation}
where $\mathcal E_0:=\frac L2\bigl[c^{2T}\|\theta_{\mathrm{init}}-\theta_0^\star\|^2
+2\eta p\,\sigma_0^2\frac{1-c^{2T}}{1-c^2}\bigr]$.
\end{theorem}

The three terms of~\eqref{eq:findbetterlabel} correspond to the error after training, the drift of the minimizer after each
deletion (Lemma~\ref{lem:shift-minimizers}), and the injected noise. Old
requests are forgotten geometrically, so the risk does not accumulate. When $g_s\le g$ and
$h_s\le h$ for all $s$,
\[
\mathcal E_r\le c^{Kr}\mathcal E_0+\frac{B\,c^{2K}}{1-c^K}\Bigl(\frac{g^2}{1-c^K}+p\,h^2\Bigr).
\]
The first term vanishes as $r$ grows and the second does not depend on $r$: the bound
reaches a plateau, which decreases geometrically with the number $K$ of corrective steps.
The proof can be found in Appendix~\ref{sec:utility-proof}.

\section{Experiments}
\label{sec:experiments}
 
\subsection{Setting}
\label{sec:exp-setting}
 
We use the digits of FEMNIST~\cite{caldas2018leaf}, in which every image is attributed to
its writer, with features from a frozen ResNet-18 pretrained on ImageNet~\cite{he2016deep,deng2009imagenet}.
The $2{,}480$ private writers form the training set
($n=279{,}306$, at most $134$ records per writer), $300$ writers are used for
validation and $300$ for testing. 
A request deletes all the records of a user, drawn
uniformly at random, so that requests are independent
(Assumption~\ref{ass:independent-requests}). The model is a ridge-regularized multiclass logistic
regression trained without noise ($\sigma_0=0$). As the contraction factor $c$ of the gradient update
(see Assumption~\ref{ass:contraction}) is given by $c=1-m/L$, $\kappa:=L/m$ represents the natural
horizon of the training algorithm. For the setup under consideration, $\kappa \approx 51$ and we report the use
of $K\in\{0.1\kappa,0.3\kappa,0.6\kappa,\kappa\}$ corrective steps per request for CoSU, where $K=0.1\kappa \approx 5$
corresponds to a very small number of fine-tuning iterations while $K=\kappa \approx 51$ comes closer to full
retraining on the retain set. All unlearning requests are certified at level $(\varepsilon,\delta)=(0.5,1/n)$.

We compare CoSU with two training-time DP learners (Section~\ref{sec:no-provisioning}), which are trained by noisy
gradient descent according to Proposition~\ref{prop:dp-provisioned} (in Appendix~\ref{app:proof-provisioning})
to satisfy Definition~\ref{def:counterfactual_sequential_unlearning} at the same privacy
level as CoSU: \emph{DP, largest group},
calibrated for the largest user, and \emph{DP, bounded contributions}, which keeps at
most $g_0$ records per user \cite{amin2019bounding,epasto2020smoothly} and is calibrated for $g_0$. We also compare with
Descent-to-Delete~\cite{neel2020descenttodeletegradientbasedmethodsmachine} and Langevin
Unlearning~\cite{chien2025langevinunlearningnewperspective}, which use the cumulative-deletion
reference. All methods are
tuned on the validation writers. Full details are given in Appendix~\ref{app:experimental-details}.

\subsection{Privacy--utility trade-off}
\label{sec:exp-tradeoff}
 
\begin{figure}[t]
    \centering
    \includegraphics[width=\textwidth]{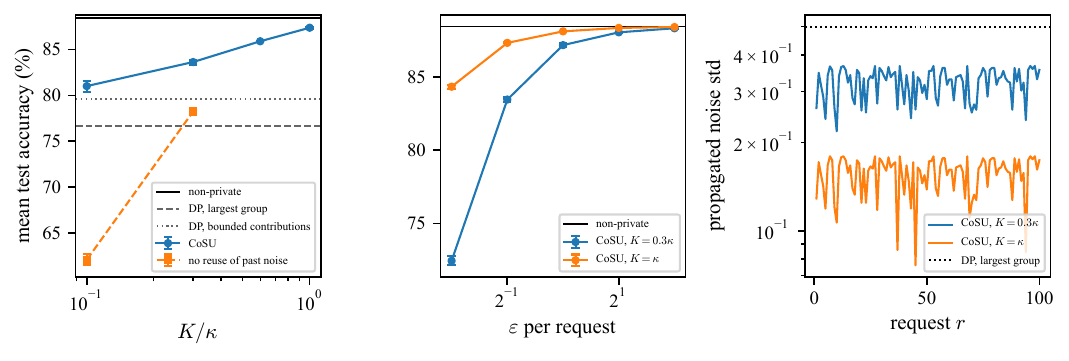}
    \caption{Privacy--utility trade-off of CoSU on $100$ writer-level requests. Left:
    accuracy averaged over the requests as a function of $K$ ($\varepsilon=0.5$),
    with and without reuse of past noise ($x_r=0$), and DP
    learners for reference. Middle: accuracy as a function of $\varepsilon$. Right:
    propagated noise standard deviation $\sqrt{W_r}$ required by CoSU (Theorem~\ref{thm:eps-delta-tracking}),
    and by DP calibrated for the largest user (Proposition~\ref{prop:dp-provisioned}).}
    \label{fig:tradeoff}
\end{figure}

\begin{table}[t]
\centering
\caption{Writer-level deletions ($100$ requests, $\varepsilon=0.5$): test accuracy (\%) before
any request, averaged over the requests (mean $\pm$ std over five runs), and at the worst
request. $^\dagger$After the last request.}
\label{tab:femnist-writers}
\begin{tabular*}{\textwidth}{@{\extracolsep{\fill}}lcccr}
\toprule
Method & before & mean & worst & steps / request \\
\midrule
Non-private (no guarantee) & 88.4 & -- & -- & -- \\
Exact retraining ($\varepsilon=0$) & -- & 88.3$^\dagger$ & -- & 200 ($100\%$) \\
\midrule
CoSU, $K=0.1\kappa$ & 88.4 & 81.0 $\pm$ 0.6 & 77.4 & 5 ($2.5\%$) \\
CoSU, $K=0.3\kappa$ & 88.4 & 83.6 $\pm$ 0.3 & 80.3 & 15 ($7.5\%$) \\
CoSU, $K=0.6\kappa$ & 88.4 & 85.9 $\pm$ $<$0.1 & 84.1 & 31 ($15.5\%$) \\
CoSU, $K=\kappa$ & 88.4 & \textbf{87.4} $\pm$ 0.1 & 86.6 & 51 ($25.5\%$) \\
\midrule
DP, largest group ($g_{\max}=134$) & 76.6 & 76.6 $\pm$ 1.1 & 76.6 & 0 \\
DP, bounded contributions ($g_0=34$) & 79.6 & 79.6 $\pm$ 1.4 & 79.6 & 0 \\
\bottomrule
\end{tabular*}
\end{table}

Figure~\ref{fig:tradeoff} and Table~\ref{tab:femnist-writers} report the accuracy of CoSU for $100$ writer-level
requests. Before any request, CoSU releases the non-private model ($88.4\%$); exact retraining maintains
this accuracy when using $T \geq 200$ retraining steps. The noise scale of a request decays
geometrically with the number of unlearning steps $K$ 
(Corollary~\ref{cor:pay-as-you-go}): at $\varepsilon=0.5$,
$K=\kappa$ ($51$ steps, $25.5\%$ of $T$) loses about one point, and
$5$ steps per request still give $81.0\%$.
As shown in the left panel, reusing the noise already injected is essential when $K$ is small: without it, the accuracy drops to
$62.1\%$ and $78.2\%$ at, respectively, $K=0.1\kappa$ and $0.3\kappa$.
Moreover, corrective steps can be traded for privacy: the noise scale of a request is proportional to
$c^{K}/\bar\mu$, so running $K=\kappa$ instead of $K=0.3\kappa$ allows $\varepsilon$ to be divided by two without losing accuracy (middle panel).
The right panel shows the propagated noise scale $\sqrt{W_r}$ required by CoSU: it follows the size of the current
request, does not grow with $r$ (Corollary~\ref{cor:pay-as-you-go}(a)), and stays about
$1.5$ and $3$ times below that of DP calibrated for the largest group at, respectively, $K=0.3\kappa$ and
$K=\kappa$.

Training-time DP requires no computation at deletion time, but loses accuracy from the
first release: it reaches $76.6\%$ when calibrated for the largest request,
and $79.6\%$ with contributions bounded to $g_0=34$ records per user (which keeps
$30\%$ of the data). CoSU matches the latter when using $5$ steps per unlearning request.

\begin{figure}[t]
    \centering
    \includegraphics[width=0.8\linewidth]{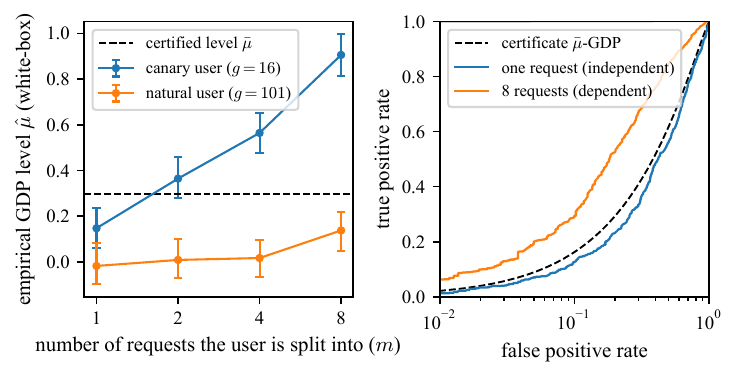}
    \caption{Membership inference on a user deleted in one request ($m=1$) or split into
    $m$ requests certified separately. Left: empirical Gaussian-DP level $\hat\mu$ of the white-box
    attack, against the certified level $\bar\mu$ (dashed). Right: trade-off curves of the attack
    on the canary user vs. certificate.}
    \label{fig:mia}
\end{figure}

\subsection{Membership inference: independent versus dependent requests}
\label{sec:exp-mia}
The certificate of a request states that the released model is indistinguishable, at
level $\bar\mu$, from the model that would have been released if the deleted group had
never been in the training set. We test this statement with a membership inference
attack (MIA). An attacker who knows the algorithm and all the other records observes the
released model and has to decide whether the records of a target user were in the
training set and then deleted, or were never there. The success of the attack is measured
by an empirical Gaussian-DP level $\hat\mu$, which cannot exceed $\bar\mu$ if the target
is protected at the certified level~\cite{dong2022gaussian}. We consider two settings. In
the first, the target is deleted in \emph{one} request, as our framework requires for the
records of a single user. In the second, its records are split into $m$ consecutive
requests, each certified at level $\bar\mu$ as if they were independent, which violates
Assumption~\ref{ass:independent-requests}. In both settings, the attack targets all the
records of the user, on the model released after the last unlearning request. We use a subset of $200$ writers, with
$K=0.1\kappa$ and $\varepsilon=1$, and run the pipeline $1{,}024$ times with the target and
$1{,}024$ times without it. The target is either a natural writer or a worst-case
\emph{canary} request (composed of $16$ identical copies of a single image with a wrong label).
The attacker has access to the parameters of the released model (white-box attack); the attack, a black-box
variant and the estimation of $\hat\mu$ are described in
Appendix~\ref{app:experimental-details}.

\begin{figure}[t]
    \centering
    \includegraphics[width=0.5\linewidth,trim=0 162 256 16,clip]{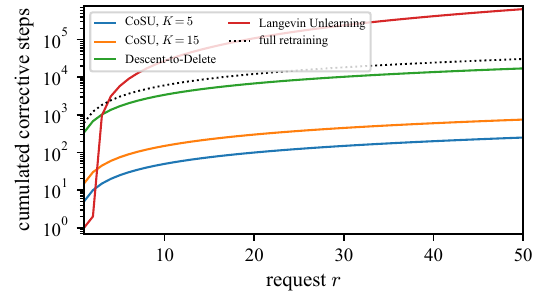}
    \caption{CoSU and certified first-order baselines ($\varepsilon=0.5$) on $50$
    single-record requests. Cumulative number of corrective steps up to request $r$; the
    dotted line is the cost of full retraining at every request.}
    \label{fig:baselines}
\end{figure}

The results are shown in Figure~\ref{fig:mia} and Table~\ref{tab:mia} in Appendix~\ref{app:experimental-details}.
When the canary is deleted in one request, the attack stays within the certificate:
$\hat\mu=0.15$ ($95\%$ interval $[0.06,0.24]$) for $\bar\mu=0.30$, and its ROC curve lies on
or below the curve allowed by the certificate. When the same user is split
into $m$ requests certified separately, every request is certified, but the user is not:
$\hat\mu$ grows to $0.36$, $0.56$ and $0.91$ for $m=2,4,8$, i.e., up to $3.1$ times
the certified level, and the AUC of the attack from $0.54$ to $0.74$. The black-box loss attack
gives the same AUCs (Table~\ref{tab:mia}). For the natural
writer, $\hat\mu$ stays between $-0.02$ and $0.14$ for every $m$: its gradients are
much smaller than the bound, so the certificate is loose for typical users. Treating
dependent requests as independent can thus invalidate the guarantee in the worst case,
whereas the attack on independent requests stays within the level certified by CoSU. 

\subsection{Comparison with sequential unlearning baselines}
\label{sec:exp-dp}

In contrast to CoSU, Descent-to-Delete and Langevin Unlearning use the cumulative-deletion
reference. Nonetheless, we tune them to the same $(\eps,\delta)$ level as CoSU, which leads to
comparable test accuracy, of about $88.5$\%, for single-record unlearning requests. However, Figure~\ref{fig:baselines}
shows that this is achieved at a much higher computational cost than for CoSU. Descent-to-Delete needs about
$340$ steps per request, while the cost of Langevin Unlearning grows by $474$ steps per request and
exceeds the cost of full retraining after the third request. By contrast, CoSU needs only $5$ steps to achieve the same
accuracy and, more importantly, its long-run computational cost stays far below that of
systematic retraining. More experimental results are reported in
Appendix~\ref{app:additional-experiments}.

\section{Conclusion}

For sequential unlearning, we proposed to certify each release against the run in which the current group never
participated, rather than against retraining without all past deletions. For independent
requests, this counterfactual reference preserves accuracy in the long run. In the smooth
strongly convex setting, CoSU implements it at a fixed cost per request, with noise set by
the request and not by its rank.
In our experiments, it requires fewer gradient steps per request than
existing certified first-order methods, and outperforms training-time DP.
The guarantee relies on the independence of the requests: membership inference attacks 
exceed the certified level when dependent requests are certified separately, and stay
within it when they are treated jointly.
Our analysis covers smooth strongly convex objectives with full-batch gradients: extending
it to stochastic gradients and possibly non-convex models is a natural next step.

\bibliographystyle{unsrtnat}
\bibliography{references2}

\clearpage
\appendix

\section{The two references}
\label{app:illustration}

Figure~\ref{fig:standard-vs-counterfactual-reference} illustrates the cumulative-deletion
reference (Section~\ref{sec:definition}) and the counterfactual reference
(Definition~\ref{def:counterfactual_sequential_unlearning}).

\begin{figure}[h]
    \centering
    \includegraphics[width=0.9\textwidth]{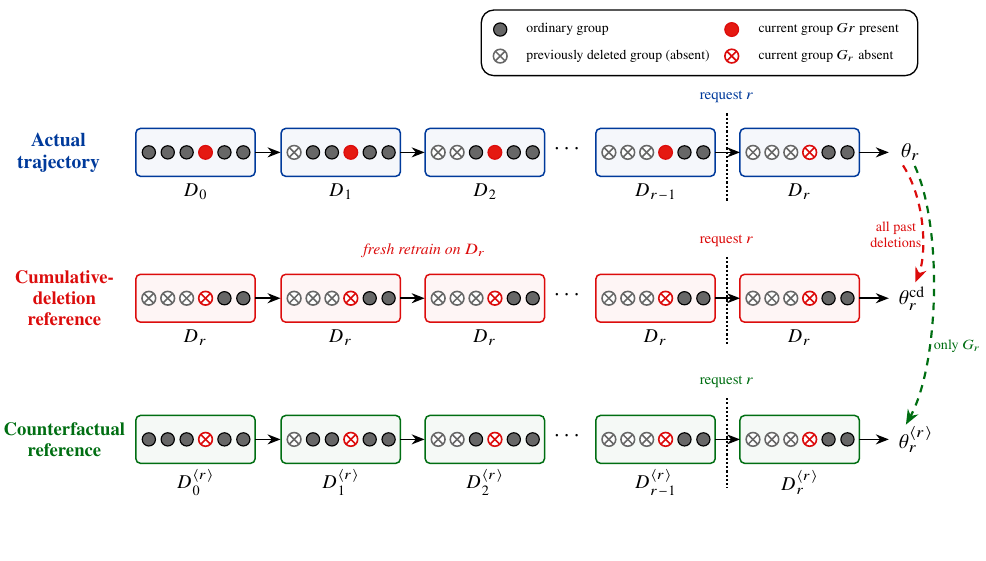}
    \caption{
    Cumulative-deletion and counterfactual references for the $r$-th sequential
    unlearning request. The cumulative-deletion reference compares the released model to
    a fresh retrain on \(D_r\), where all past deletions are already absent. The
    counterfactual reference runs the actual sequential pipeline, with only the current
    group \(G_r\) absent from the start.
    }
    \label{fig:standard-vs-counterfactual-reference}
\end{figure}

\section{Privacy preliminaries}
\label{app:privacy-preliminaries}

For two distributions \(P,Q\),
the trade-off function is
\[
\mathcal T(P,Q)(\alpha)
:=
\inf_{\phi:\, \mathbb E_P[\phi]\le \alpha}
\bigl(1-\mathbb E_Q[\phi]\bigr),
\qquad \alpha\in[0,1],
\]
where the infimum is over all randomized tests \(0\le \phi\le1\). Larger
trade-off functions correspond to harder distinguishability.

For \(\mu\ge0\), let
\[
\mathcal G_\mu(\alpha)
:=
\Phi\!\left(\Phi^{-1}(1-\alpha)-\mu\right),
\]
which is the trade-off function between \(\mathcal N(0,1)\) and
\(\mathcal N(\mu,1)\). We say that an ordered pair \((P,Q)\) satisfies
\(\mu\)-GDP if
\[
\mathcal T(P,Q)\ge \mathcal G_\mu.
\]

We use the following sufficient conversion from GDP to DP.

\begin{lemma}[Sufficient GDP-to-DP conversion]
\label{lem:gdp-to-dp}
Let \(\varepsilon>0\) and \(\delta\in(0,1)\). If \((P,Q)\) satisfies
\(\mu\)-GDP with
\[
\mu\le
\bar\mu_{\varepsilon,\delta}
:=
\sqrt{2\log(1/\delta)+2\varepsilon}
-
\sqrt{2\log(1/\delta)},
\]
then \(P\) and \(Q\) are \((\varepsilon,\delta)\)-indistinguishable.
\end{lemma}

\begin{proof}
Let
\[
\delta_{P,Q}(\varepsilon)
:=
\sup_A \{Q(A)-e^\varepsilon P(A)\}.
\]
By the standard variational relation between trade-off functions and
hockey-stick divergences, \(\mathcal T(P,Q)\ge \mathcal G_\mu\) implies
\[
\delta_{P,Q}(\varepsilon)
\le
\Phi\!\left(\frac{\mu}{2}-\frac{\varepsilon}{\mu}\right)
-
e^\varepsilon
\Phi\!\left(-\frac{\mu}{2}-\frac{\varepsilon}{\mu}\right)
\le
\Phi\!\left(\frac{\mu}{2}-\frac{\varepsilon}{\mu}\right).
\]
Let \(a=\sqrt{2\log(1/\delta)}\). The condition
\[
\mu\le \sqrt{a^2+2\varepsilon}-a
\]
is equivalent to
\[
\frac{\mu}{2}-\frac{\varepsilon}{\mu}\le -a .
\]
Therefore,
\[
\delta_{P,Q}(\varepsilon)
\le
\Phi(-a)
\le
e^{-a^2/2}
=
\delta ,
\]
which proves the claim. Since $\Phi(-a)\le\frac12e^{-a^2/2}$, the lemma is conservative
by a factor of two in $\delta$; the exact conversion of Dong et al.~\cite{dong2022gaussian}
removes this slack.
\end{proof}

Thus, throughout the proof it is enough to construct a GDP certificate with
parameter at most \(\bar\mu_{\varepsilon,\delta}\).

\section{Proofs and additional results for Section~\ref{sec:definition}}
\label{app:proofs-definition}

\subsection{Proofs of Propositions~\ref{prop:independence-reference} and~\ref{prop:persistence-postprocessing}}

\begin{proof}[Proof of Proposition~\ref{prop:independence-reference}]
Since the groups are disjoint, the other groups are contained in $D\setminus G_r$, and every
counterfactual dataset is obtained from $D\setminus G_r$ by removing some of them: for $s<r$,
\[
D_s^{\langle r\rangle}=\bigl(D\setminus G_r\bigr)\setminus(G_1\cup\cdots\cup G_s),
\]
so none of them involves the content of $G_r$. We conclude by induction on $s$. For $s=0$,
$\theta_0^{\langle r\rangle}\sim\A(D\setminus G_r)$, and $\A$ does not depend on the data
(Assumption~\ref{ass:independence}). If the claim holds at $s-1$, then
$\theta_s^{\langle r\rangle}\sim\U_s(\theta_{s-1}^{\langle r\rangle},D_s^{\langle r\rangle})$:
the kernel $\U_s$ depends on the requests only through their sizes and levels
(Assumption~\ref{ass:independence}), which do not change when the content of $G_r$ is
replaced; its two arguments do not involve this content (induction hypothesis and the
display above); and, by Assumption~\ref{ass:stateless}, no other data-dependent state is
carried across rounds, the randomness being fresh at every round. The claim follows at $s$,
hence at $s=r$.
\end{proof}

\begin{proof}[Proof of Proposition~\ref{prop:persistence-postprocessing}]
Fix $t>r$. From round $r$ on, both trajectories use the same datasets:
$D_s^{\langle r\rangle}=D_s$ for every $s\ge r$, since $G_r$ is absent from both after request
$r$. The later groups $G_{r+1},\ldots,G_t$ are the same in both worlds, since they are
independent of the released models (Assumption~\ref{ass:adaptive}). The kernels
$\U_{r+1},\ldots,\U_t$ are the same (Assumption~\ref{ass:independence}), and no other
data-dependent state is carried across rounds (Assumption~\ref{ass:stateless}). The updates
from round $r$ to round $t$ therefore define a common Markov kernel
\[
K_{r:t}:=\U_t(\cdot,D_t)\circ\cdots\circ\U_{r+1}(\cdot,D_{r+1}),
\] such that 
\[
\theta_t\sim K_{r:t}(\theta_r),\quad
\theta_t^{\langle r\rangle}\sim K_{r:t}(\theta_r^{\langle r\rangle}),
\]
and the post-processing property of differential
privacy~\cite[Proposition~2.1]{dwork2014algorithmic}, $X\simeq_{\varepsilon,\delta}X'
\Rightarrow K(X)\simeq_{\varepsilon,\delta}K(X')$ for every such common kernel $K$, gives
the claim.
\end{proof}

\subsection{Independent and dependent requests}

\paragraph{Request process}
Propositions~\ref{prop:semantics} and~\ref{prop:necessity} concern random requests. Fix a
rank $r$, a group $G_r\subseteq D$ and a level $(\varepsilon_r,\delta_r)$. The \emph{other
requests} $\rho:=(G_s,\varepsilon_s,\delta_s)_{s\neq r}$ are drawn by a \emph{request
mechanism}: a Markov kernel that takes the training set as input and whose randomness is
independent of that of $\A$ and $(\U_s)_{s\ge1}$ (Assumption~\ref{ass:adaptive}). In the
\emph{actual process}, the training set is $D$, the other requests are $\rho$, and request
$r$ deletes $G_r$. In the \emph{process in which $G_r$ was never in the training set}, the
training set is $D\setminus G_r$, the other requests $\rho^\circ$ are drawn by the same
mechanism run on $D\setminus G_r$, and request $r$ is still issued, as a void request with
the same size and privacy level. Assumption~\ref{ass:independent-requests} states that $\rho$ and
$\rho^\circ$ have the same law, which does not depend on the content of $G_r$.

\begin{proposition}[Protection of independent groups]
\label{prop:semantics}
Let Assumptions~\ref{ass:independence}--\ref{ass:stateless} and~\ref{ass:independent-requests}
hold, and let the pipeline
satisfy Definition~\ref{def:counterfactual_sequential_unlearning} at every request.
Fix a request $r$ and a time $t\ge r$. Let $\theta_t^\circ$ and $\pi_t^\circ$ be the
release and the sizes and privacy levels of the requests up to $t$ in the process in which $G_r$
was never in the training set, request $r$ being still issued as a void request, and let
$\pi_t:=(g_s,\varepsilon_s,\delta_s)_{s\le t}$ be their counterparts in the actual process.
Then
\[
(\theta_t,\pi_t)\ \simeq_{\varepsilon_r,\delta_r}\ (\theta_t^\circ,\pi_t^\circ).
\]
Conditionally on the other requests, $\theta_t\simeq_{\varepsilon_r,\delta_r}\theta_t^\circ$
holds even without Assumption~\ref{ass:independent-requests}.
\end{proposition}

\begin{proof}[Proof of Proposition~\ref{prop:semantics}]
Fix a request $r$, a time $t\ge r$ and the size and privacy level of request $r$. Let $\rho$
denote the other requests (their groups, sizes and privacy levels). The public information
$\pi_t$ is a function $\pi(\rho)$ of $\rho$ and of the size and privacy level of request $r$, and
the same function gives $\pi_t^\circ$ from the other requests $\rho^\circ$ of the process
without $G_r$, since the void request keeps its size and privacy level. By
Assumption~\ref{ass:adaptive}, the requests are independent of the internal randomness of
the kernels, so we may condition on $\rho$ without changing the law of this randomness.

\emph{Conditional statement.} Fix $\rho$ in the support of the request process. In the actual
process, $\theta_t$ is produced by the pipeline on $D$ with this sequence. In the process
without $G_r$, the other requests are $\rho$ and request $r$ is void with the same size
and privacy level, so the kernels are the same (Assumption~\ref{ass:independence}) and the dataset
used at round $s$ is $D_s\setminus G_r=D_s^{\langle r\rangle}$. This is the counterfactual
trajectory, so $\theta_t^\circ\mid\rho$ has the law of $\theta_t^{\langle r\rangle}\mid\rho$.
Definition~\ref{def:counterfactual_sequential_unlearning} gives
$\theta_r\simeq_{\varepsilon_r,\delta_r}\theta_r^{\langle r\rangle}$ given $\rho$, and
Proposition~\ref{prop:persistence-postprocessing} extends it to time $t$. This does not use
Assumption~\ref{ass:independent-requests}.

\emph{Averaging.} For a measurable set $S$ of pairs, write $S_\pi:=\{\theta:(\theta,\pi)\in S\}$.
Then
\begin{align*}
\Pr\bigl((\theta_t,\pi_t)\in S\bigr)
&=\E_\rho\bigl[\Pr(\theta_t\in S_{\pi(\rho)}\mid\rho)\bigr]\\
&\le\E_\rho\bigl[e^{\varepsilon_r}\Pr(\theta_t^\circ\in S_{\pi(\rho)}\mid\rho)+\delta_r\bigr]\\
&=e^{\varepsilon_r}\Pr\bigl((\theta_t^\circ,\pi_t^\circ)\in S\bigr)+\delta_r,
\end{align*}
where the last equality uses that $\rho$ and $\rho^\circ$ have the same law given the size
and privacy level of request $r$ (Assumption~\ref{ass:independent-requests}). The reverse
inequality is identical.
\end{proof}

\begin{proposition}[Dependent requests must be merged]
\label{prop:necessity}
If request $q$ is issued if and only if $G_r$ is in the training set, then, for any
pipeline, the actual process and the process in which $G_r$ was never in the training set
are perfectly distinguishable from what is public about the requests alone (their number,
sizes and levels). If the other requests depend neither on $G_r$ nor on $G_q$, this public
information has the same law in the actual process and in the process in which $G_r\cup G_q$ was never in the training set.
\end{proposition}

\begin{proof}[Proof of Proposition~\ref{prop:necessity}]
In the actual process, the records of $G_r$ are in the training set, so request $q$ is
issued, which is public. In the process in which $G_r$ was never in the training set, the
request mechanism is run on a training set without $G_r$, so request $q$ is not issued;
only request $r$ is kept, as a void request with the same size and level. The event
``request $q$ is issued'' has probability one in the first process and zero in the second,
so no pair $(\varepsilon,\delta)$ with $\delta<1$ can make the two processes
indistinguishable, whatever the released models. In the process in which $G_r\cup G_q$
never entered the training set, both requests are kept as void requests with the same
sizes and privacy levels, and, if the other requests depend neither on $G_r$ nor on $G_q$, the
number, sizes and privacy levels of the requests have the same law in both processes.
\end{proof}

\subsection{Training-time differential privacy}

\begin{proposition}[DP and provisioning]
\label{prop:no-provisioning}
If $\U_s(\theta,\cdot)$ is the point mass at $\theta$ for every $s$, the pipeline satisfies
Definition~\ref{def:counterfactual_sequential_unlearning} if and only if
$\A(D)\simeq_{\varepsilon,\delta}\A(D\setminus G)$ for every $D$ and every request
$(G,\varepsilon,\delta)$. In particular, a learner that is
$(\varepsilon,\delta)$-DP for groups of at most $g_{\max}$ records satisfies it for all
requests of size at most $g_{\max}$ and privacy level at least $(\varepsilon,\delta)$.
\end{proposition}

\begin{proof}[Proof of Proposition~\ref{prop:no-provisioning}]
If no update is performed, $\theta_r=\theta_0\sim\A(D)$ and, for the counterfactual
trajectory of request $r$, $\theta_r^{\langle r\rangle}=\theta_0^{\langle r\rangle}\sim
\A(D\setminus G_r)$. The definition therefore holds at request $r$ if and only if
$\A(D)\simeq_{\varepsilon_r,\delta_r}\A(D\setminus G_r)$, which gives the equivalence.
If $\A$ is $(\varepsilon,\delta)$-DP for groups of at most $g_{\max}$ records, then
$\A(D)\simeq_{\varepsilon,\delta}\A(D\setminus G_r)$ whenever $g_r\le g_{\max}$, and this
implies the guarantee at any level $(\varepsilon_r,\delta_r)\ge(\varepsilon,\delta)$; no
parameter accumulates since each request is compared with its own reference.
\end{proof}

\section{Dependent requests arriving at different times}
\label{app:dependent}
 
When dependent requests are issued at different times, they cannot be merged into a
single request. Let the requests be partitioned into public \emph{blocks}, such that
Assumption~\ref{ass:independent-requests} holds for blocks instead of single requests (the
requests of a block may depend on each other; requests of different blocks may not), and
let $\mathcal B_r\subseteq\{1,\ldots,r\}$ be the set of requests of the block of $r$ issued
up to $r$, which contains $r$. Let $G_{\mathcal B_r}:=\bigcup_{s\in\mathcal B_r}G_s$. The
\emph{joint counterfactual reference} of request $r$ runs the same pipeline, with the same
kernels, on the data from which $G_{\mathcal B_r}$ has been removed from the start:
\[
\theta_0^{\langle\mathcal B_r\rangle}\sim\A\bigl(D_0\setminus G_{\mathcal B_r}\bigr),
\qquad
\theta_s^{\langle\mathcal B_r\rangle}\mid\theta_{s-1}^{\langle\mathcal B_r\rangle}
\sim\U_s\bigl(\theta_{s-1}^{\langle\mathcal B_r\rangle},D_s\setminus G_{\mathcal B_r}\bigr),
\quad s\ge1,
\]
the requests of $\mathcal B_r$ being void. The requirement becomes
\[
\theta_r\simeq_{\varepsilon_s,\delta_s}\theta_r^{\langle\mathcal B_r\rangle}
\qquad\text{for every } s\in\mathcal B_r .
\]
When every block is a single request, $\mathcal B_r=\{r\}$ and this is
Definition~\ref{def:counterfactual_sequential_unlearning}.
Propositions~\ref{prop:independence-reference}
and~\ref{prop:persistence-postprocessing} hold with the same proofs, $G_r$ being replaced
by $G_{\mathcal B_r}$.
 
\paragraph{Calibration}
The kernels may now depend on the partition into blocks, which is public. For
$0\le q\le r$, let
\[
A_r(q):=\sum_{s\in\mathcal B_r}g_s\,c^{K(r-q_s)}I_{q_s},\qquad q_s:=\min(q,s-1),
\]
be the influence accumulated by the block up to the end of episode $q$ (a group $G_s$ is
present until the end of episode $s-1$), propagated to the end of episode $r$. The
calibration of Theorem~\ref{thm:eps-delta-tracking} is modified as follows:
\[
\bar\mu_r^{\mathcal B}:=\min_{s\in\mathcal B_r}\bar\mu_s,\quad
b_r:=A_r(r),\quad
\bar x_r:=P_r\min_{0\le q<r,\ W_q>0}\frac{A_r(q)}{c^{2K(r-q)}W_q},\quad
x_r:=\min\Bigl\{\frac{(\bar\mu_r^{\mathcal B})^2P_r}{b_r},\ \bar x_r\Bigr\},
\]
with $x_r:=0$ if $P_r=0$; $P_r=c^{2K}W_{r-1}$ and
$W_r=P_r+2\eta\sigma_r^2S_2(K)$ are unchanged. For $\mathcal B_r=\{r\}$,
$A_r(q)=g_rc^{K(r-q)}I_q$ for $q<r$ and $A_r(r)=g_rc^KI_{r-1}$, which are the quantities of
Section~\ref{sec:theory}.
 
\begin{proposition}[Calibration for blocks]
\label{prop:blocks}
Let Assumptions~\ref{ass:independence}, \ref{ass:adaptive}, \ref{ass:stateless},
\ref{ass:single-release} and~\ref{ass:contraction} hold, and let
\[
\sigma_r^2=\frac{(b_r-x_r)^2}{2\eta S_2(K)\,\bigl((\bar\mu_r^{\mathcal B})^2-x_r^2/P_r\bigr)}
\qquad(\sigma_r:=0\ \text{if}\ x_r=b_r),
\]
with the quantities above. Then, for every dataset, every request sequence and every
$r$,
$\mathcal T\bigl(\theta_r,\theta_r^{\langle\mathcal B_r\rangle}\bigr)\ge
\mathcal G_{\bar\mu_r^{\mathcal B}}$, and therefore
$\theta_r\simeq_{\varepsilon_s,\delta_s}\theta_r^{\langle\mathcal B_r\rangle}$ for every
$s\in\mathcal B_r$. Moreover,
$2\eta\sigma_r^2S_2(K)\le b_r^2/(\bar\mu_r^{\mathcal B})^2$.
\end{proposition}
 
\begin{proof}
We follow the proof of Theorem~\ref{thm:eps-delta-tracking}
(Appendix~\ref{sec:appendix-proof-CoSU}) and indicate the changes. We apply
Theorem~\ref{thm:bok-shifted-certificate} with $N=N_r$ to the actual trajectory and to the
joint counterfactual trajectory; they start from the same point and use the same schedule
and the same variances, which are functions of public quantities only.
 
\emph{Discrepancy.} At step $k$, the two gradient maps are those of two objectives that
differ by the losses of the groups $G_s$, $s\in\mathcal B_r$, still present in the actual
training set, that is, such that $k<H_s$. Both maps are $c$-Lipschitz
(Assumption~\ref{ass:contraction}), so~\eqref{eq:CoSU-sk} becomes
\[
s_k=\eta M\sum_{s\in\mathcal B_r:\ k<H_s}g_s ,\qquad 0\le k<N_r .
\]
Since $H_s=N_{s-1}$, the sum is over all $s\in\mathcal B_r$ during training and over the
$s\in\mathcal B_r$ with $s>j$ during episode $j$: $s_k$ is constant on each phase, and
vanishes on episode $r$. For $0\le q\le r$, since $\min(N_q,H_s)=N_{q_s}$ and
$\eta M\sum_{\ell<N_j}c^{N_r-1-\ell}=c^{K(r-j)}I_j$,
\[
\sum_{\ell=0}^{N_q-1}c^{N_r-1-\ell}s_\ell
=\eta M\sum_{s\in\mathcal B_r}g_s\sum_{\ell<\min(N_q,H_s)}c^{N_r-1-\ell}
=\sum_{s\in\mathcal B_r}g_s\,c^{K(r-q_s)}I_{q_s}=A_r(q),
\]
which replaces~\eqref{eq:CoSU-available}. In particular $A_r(r-1)=A_r(r)=b_r$, since
$q_s=s-1$ for every $s\in\mathcal B_r$ as soon as $q\ge r-1$: the terminal
condition~\eqref{eq:CoSU-linear} holds with this $b_r$, and~\eqref{eq:CoSU-partial} is
unchanged.
 
\emph{Allocation and admissibility.} We keep the allocation~\eqref{eq:CoSU-profile}, so
that~\eqref{eq:CoSU-two-param} holds with the new $b_r$. The proof of
Lemma~\ref{lem:admissible} only uses that $s_\ell$ is constant on each phase, which still
holds: it is enough to check~\eqref{eq:CoSU-partial} at the ends of the phases. With the
notation of that proof, for $q<r$,
\[
F(N_q)=A_r(q)-\frac{x}{P_r}\,c^{2K(r-q)}W_q ,
\]
which is nonnegative when $x\le\bar x_r$, and $F(N_r)=b_r-x-\gamma Q_r=0$.
 
\emph{Budget.} The term $q=r-1$ of the cap is $P_rA_r(r-1)/(c^{2K}W_{r-1})=b_r$, so
$x_r\le b_r$. The proof of Theorem~\ref{thm:eps-delta-tracking} only uses $x_r\le b_r$ and
$x_r\le\bar\mu^2P_r/b_r$; with $\bar\mu:=\bar\mu_r^{\mathcal B}$, it shows that $\sigma_r$ is
well defined and that the allocation is admissible with $\mu\le\bar\mu_r^{\mathcal B}$.
Theorem~\ref{thm:bok-shifted-certificate} gives the Gaussian-DP bound, and
Lemma~\ref{lem:gdp-to-dp} gives the level $(\varepsilon_s,\delta_s)$ of every
$s\in\mathcal B_r$, since $\bar\mu_r^{\mathcal B}\le\bar\mu_s$. The last bound is the
argument of Corollary~\ref{cor:pay-as-you-go}(a): $Q_r$ is nonincreasing in $x$ on
$[0,x_r]$, so $Q_r\le b_r^2/(\bar\mu_r^{\mathcal B})^2$.
\end{proof}
 
\paragraph{Price of a joint counterfactual certificate}
Let $\sigma_0=0$, and let all requests have a common size $g$ and level $\bar\mu$ and form
a single block. Then $b_r=g\sum_{s\le r}c^{K(r-s+1)}I_{s-1}$ tends to
$gc^KI_\infty/(1-c^K)$, with $I_\infty:=\eta M/(1-c)$, instead of $gc^KI_\infty$ for
independent requests. When the cap is not active, which is what we observe numerically
in this case, $W_r=b_r^2/\bar\mu^2$ as in the proof of
Corollary~\ref{cor:pay-as-you-go}(b), so the propagated noise standard deviation is
asymptotically $1/(1-c^K)$ times that of independent requests: this is the price of
treating requests jointly. This joint certificate is still counterfactual: its reference
replays the pipeline without the block, and is not the fresh retraining $\A(D_r)$ of the
cumulative-deletion reference (Definition~\ref{def:naive_sequential_unlearning}). The numerical check of
Appendix~\ref{app:additional-experiments} covers this extension.

\section{Proofs of Theorem~\ref{thm:eps-delta-tracking}, Corollary~\ref{cor:pay-as-you-go} and Proposition~\ref{prop:training-noise}}
\label{sec:appendix-proof-CoSU}

\subsection{Shifted-interpolation with varying noise scales}

\begin{theorem}[Shifted-interpolation, after Bok et al.~\cite{bok2024shiftedinterpolationdifferentialprivacy}]
\label{thm:bok-shifted-certificate}
Consider two noisy trajectories
\[
X_{k+1}=\varphi_{k+1}(X_k)+\xi_{k+1},
\qquad
X'_{k+1}=\varphi'_{k+1}(X'_k)+\xi'_{k+1},
\qquad
0\le k<N,
\]
started from the same point, where
\[
\xi_{k+1},\xi'_{k+1}\sim
\mathcal N(0,2\eta\sigma_{k+1}^2I),\qquad\sigma_{k+1}\ge0,
\]
are independent of the past of their trajectory.
Assume that \(\varphi_{k+1}\) and \(\varphi'_{k+1}\) are \(c\)-Lipschitz, and let
\[
\sup_\theta
\norm{\varphi_{k+1}(\theta)-\varphi'_{k+1}(\theta)}
\le s_k.
\]
If there exist deterministic sequences \(z_k\ge0\) and \(a_{k+1}\ge0\) such that
\begin{equation}
\label{eq:bok-recursion}
z_0=0,
\qquad
z_{k+1}=cz_k+s_k-a_{k+1},
\qquad
0\le a_{k+1}\le cz_k+s_k,
\qquad
z_N=0,
\end{equation}
with $a_{k+1}=0$ whenever $\sigma_{k+1}=0$, then
\begin{equation}
\label{eq:bok-gdp}
\mathcal T(X_N,X'_N)
\ge
\mathcal G_\mu,
\qquad
\mu^2
:=
\sum_{k:\,\sigma_{k+1}>0}
\frac{a_{k+1}^2}{2\eta\sigma_{k+1}^2}.
\end{equation}
\end{theorem}
 
Each step with Gaussian noise of variance $2\eta\sigma_{k+1}^2$ absorbs a shift $a_{k+1}$ at
cost $a_{k+1}^2/(2\eta\sigma_{k+1}^2)$, the remaining discrepancy $z_k$ being propagated by
the $c$-Lipschitz maps. Theorem~C.3 of Bok et
al.~\cite{bok2024shiftedinterpolationdifferentialprivacy} is the case of a constant noise
scale. We need noise scales that change from one step to the next, and steps without
noise, on which no shift may be placed (the condition $a_{k+1}=0$); we therefore give the
proof, which follows theirs.
 
\paragraph{Trade-off functions}
For $f=\mathcal T(P,Q)$ and $g=\mathcal T(P',Q')$, the tensor product is
$f\otimes g:=\mathcal T(P\times P',Q\times Q')$, and $\mathrm{Id}:=\mathcal T(P,P)$. We use
the following properties \cite{dong2022gaussian}; (P1) and (P2) are Lemmas~2.5 and~2.7 of
Bok et al.~\cite{bok2024shiftedinterpolationdifferentialprivacy}.
\begin{itemize}
\item[(P1)] \emph{Post-processing.} $\mathcal T(h(X),h(X'))\ge\mathcal T(X,X')$ for every
measurable map $h$.
\item[(P2)] \emph{Composition.} If $Y\mid X=x\sim\mathsf K(x)$ and
$Y'\mid X'=x\sim\mathsf K'(x)$ for two Markov kernels such that
$\mathcal T(\mathsf K(x),\mathsf K'(x))\ge g$ for every $x$, then
$\mathcal T\bigl((X,Y),(X',Y')\bigr)\ge\mathcal T(X,X')\otimes g$.
\item[(P3)] \emph{Tensor product.} It is commutative, associative and nondecreasing in
each argument, and $f\otimes\mathrm{Id}=f$.
\item[(P4)] \emph{Gaussian noise.} For $v>0$,
$\mathcal T\bigl(\mathcal N(u,vI),\mathcal N(0,vI)\bigr)=\mathcal G_{\|u\|/\sqrt v}$, and
$\mathcal G_{\mu_1}\otimes\mathcal G_{\mu_2}=\mathcal G_{\sqrt{\mu_1^2+\mu_2^2}}$.
\end{itemize}
 
\begin{lemma}[One noisy step]
\label{lem:one-noisy-step}
Let $Y$ and $V$ be random vectors defined on the same probability space, with
$\|V\|\le a$ almost surely, and let $Y'$ be a random vector. Let $v>0$ and
$\xi,\xi'\sim\mathcal N(0,vI)$, with $\xi$ independent of $(Y,V)$ and $\xi'$ independent
of $Y'$. Then
\[
\mathcal T\bigl(Y+V+\xi,\ Y'+\xi'\bigr)\ \ge\ \mathcal T(Y,Y')\otimes\mathcal G_{a/\sqrt v}.
\]
\end{lemma}
 
\begin{proof}
Given $Y=y$, the law $\mathsf K(y)$ of $V+\xi$ is the mixture
$\int\mathcal N(u,vI)\,\nu_y(du)$, where $\nu_y$ is the conditional law of $V$, supported
on the ball of radius $a$. Given $Y'=y$, the law of $\xi'$ is
$\mathsf K'(y):=\mathcal N(0,vI)$. Let $\phi$ be a test with
$\mathbb E_{\mathsf K(y)}[\phi]\le\alpha$. Since
$\mathbb E_{\mathsf K(y)}[\phi]=\int\mathbb E_{\mathcal N(u,vI)}[\phi]\,\nu_y(du)$, there
exists $u$ with $\|u\|\le a$ and $\mathbb E_{\mathcal N(u,vI)}[\phi]\le\alpha$. By (P4),
and since trade-off functions are nonincreasing and $\mathcal G_\mu$ decreases with $\mu$,
\[
1-\mathbb E_{\mathcal N(0,vI)}[\phi]\ \ge\ \mathcal G_{\|u\|/\sqrt v}(\alpha)\ \ge\
\mathcal G_{a/\sqrt v}(\alpha).
\]
Hence $\mathcal T(\mathsf K(y),\mathsf K'(y))\ge\mathcal G_{a/\sqrt v}$ for every $y$, and
(P2) gives
$\mathcal T\bigl((Y,V+\xi),(Y',\xi')\bigr)\ge\mathcal T(Y,Y')\otimes\mathcal G_{a/\sqrt v}$.
(P1) with $h(y,w):=y+w$ concludes.
\end{proof}
 
\begin{proof}[Proof of Theorem~\ref{thm:bok-shifted-certificate}]
Let $v_{k+1}:=2\eta\sigma_{k+1}^2$, and let
$\lambda_{k+1}:=a_{k+1}/(cz_k+s_k)\in[0,1]$, with $\lambda_{k+1}:=0$ if $cz_k+s_k=0$. On
the probability space of $(X_k)_k$, define the interpolated sequence
\[
\tilde X_0:=X_0,\qquad
\tilde X_{k+1}:=\varphi'_{k+1}(\tilde X_k)+V_{k+1}+\xi_{k+1},\qquad
V_{k+1}:=\lambda_{k+1}\bigl(\varphi_{k+1}(X_k)-\varphi'_{k+1}(\tilde X_k)\bigr),
\]
which uses the noise of the first trajectory.
 
\emph{Step 1: $\|X_k-\tilde X_k\|\le z_k$ and $\|V_{k+1}\|\le a_{k+1}$ almost surely.} This
holds at $k=0$ since $z_0=0$. If $\|X_k-\tilde X_k\|\le z_k$, then
\[
\bigl\|\varphi_{k+1}(X_k)-\varphi'_{k+1}(\tilde X_k)\bigr\|
\le\bigl\|\varphi_{k+1}(X_k)-\varphi_{k+1}(\tilde X_k)\bigr\|
+\bigl\|\varphi_{k+1}(\tilde X_k)-\varphi'_{k+1}(\tilde X_k)\bigr\|
\le cz_k+s_k ,
\]
so $\|V_{k+1}\|\le\lambda_{k+1}(cz_k+s_k)=a_{k+1}$, and, since
$X_{k+1}-\tilde X_{k+1}=(1-\lambda_{k+1})\bigl(\varphi_{k+1}(X_k)-\varphi'_{k+1}(\tilde X_k)\bigr)$,
\[
\|X_{k+1}-\tilde X_{k+1}\|\le(1-\lambda_{k+1})(cz_k+s_k)=cz_k+s_k-a_{k+1}=z_{k+1}.
\]
 
\emph{Step 2: one step.} If $\sigma_{k+1}>0$, we apply Lemma~\ref{lem:one-noisy-step} with
$Y:=\varphi'_{k+1}(\tilde X_k)$, $V:=V_{k+1}$, $Y':=\varphi'_{k+1}(X'_k)$, $a:=a_{k+1}$ and
$v:=v_{k+1}$; $\xi_{k+1}$ is independent of $(X_k,\tilde X_k)$, hence of $(Y,V)$. With (P1)
for the map $\varphi'_{k+1}$ and (P3),
\[
\mathcal T(\tilde X_{k+1},X'_{k+1})
\ge\mathcal T\bigl(\varphi'_{k+1}(\tilde X_k),\varphi'_{k+1}(X'_k)\bigr)
\otimes\mathcal G_{a_{k+1}/\sqrt{v_{k+1}}}
\ge\mathcal T(\tilde X_k,X'_k)\otimes\mathcal G_{a_{k+1}/\sqrt{v_{k+1}}} .
\]
If $\sigma_{k+1}=0$, then $a_{k+1}=0$, so $V_{k+1}=0$ and there is no noise:
$\tilde X_{k+1}=\varphi'_{k+1}(\tilde X_k)$ and $X'_{k+1}=\varphi'_{k+1}(X'_k)$, and (P1)
gives $\mathcal T(\tilde X_{k+1},X'_{k+1})\ge\mathcal T(\tilde X_k,X'_k)$.
 
\emph{Step 3: conclusion.} Since $\tilde X_0=X_0$ and $X'_0$ are the same point,
$\mathcal T(\tilde X_0,X'_0)=\mathrm{Id}$. By induction on $k$ with Step~2, (P3) and (P4),
\[
\mathcal T(\tilde X_N,X'_N)\ \ge\
\bigotimes_{k:\,\sigma_{k+1}>0}\mathcal G_{a_{k+1}/\sqrt{v_{k+1}}}=\mathcal G_\mu .
\]
Finally, $z_N=0$ and Step~1 give $\tilde X_N=X_N$ almost surely, which
is~\eqref{eq:bok-gdp}.
\end{proof}
 
The bound is stated for the ordered pair $(X_N,X'_N)$. The Gaussian trade-off function is
symmetric, so $\mathcal T(P,Q)\ge\mathcal G_\mu$ implies $\mathcal T(Q,P)\ge\mathcal G_\mu$,
which gives the two inequalities defining $\simeq_{\varepsilon,\delta}$.

\subsection{Reduction to a one-parameter allocation}

Fix a request $r$ and apply Theorem~\ref{thm:bok-shifted-certificate} with \(N=N_r\),
\(X_N=\theta_{N_r}\) and \(X'_N=\theta_{N_r}^{\langle r\rangle}\). Both trajectories start from
$\theta_{\mathrm{init}}$ and use the same public schedule \((T,K)\) and the same variances,
$2\eta\sigma_0^2$ during training and $2\eta\sigma_j^2$ during episode $j$, which are
functions of the sizes and privacy levels of the requests only (Assumption~\ref{ass:independence}).

\paragraph{Discrepancy}
Before time \(H_r=N_{r-1}\), the actual gradient map is that of an objective containing the
records of $G_r$, and the counterfactual one is that of the same objective without them;
the other groups are removed at the same times in both trajectories. By
Assumption~\ref{ass:contraction}, both maps are $c$-Lipschitz, and
\(
\sup_\theta\|\eta\sum_{d\in G_r}\nabla\ell(\theta,d)\|\le \eta M g_r .
\)
From time $H_r$ on, both trajectories use $f_r$. Hence
\begin{equation}
\label{eq:CoSU-sk}
s_k=\eta Mg_r
\quad (0\le k<H_r),
\qquad
s_k=0
\quad (H_r\le k<N_r).
\end{equation}
Unrolling \eqref{eq:bok-recursion} gives
\begin{equation}
\label{eq:CoSU-unrolled}
c^{N_r-k}z_k
=
\sum_{\ell=0}^{k-1}c^{N_r-1-\ell}s_\ell
-
\sum_{\ell=0}^{k-1}c^{N_r-1-\ell}a_{\ell+1}.
\end{equation}
Since $\eta M\sum_{\ell<N_q}c^{N_r-1-\ell}=c^{K(r-q)}I_q$ for $0\le q<r$, the
discrepancy accumulated up to the end of episode $q$ (of training for $q=0$) and
propagated to $N_r$ is
\begin{equation}
\label{eq:CoSU-available}
\sum_{\ell=0}^{N_q-1}c^{N_r-1-\ell}s_\ell=g_r\,c^{K(r-q)}I_q .
\end{equation}
In particular, the terminal condition $z_{N_r}=0$ reads
\begin{equation}
\label{eq:CoSU-linear}
\sum_{k=0}^{N_r-1}c^{N_r-1-k}a_{k+1}=g_rc^KI_{r-1}=b_r,
\end{equation}
and the conditions $z_k\ge0$ for all $k$ are equivalent to
\begin{equation}
\label{eq:CoSU-partial}
\sum_{\ell=0}^{k-1}c^{N_r-1-\ell}a_{\ell+1}
\le
\sum_{\ell=0}^{k-1}c^{N_r-1-\ell}s_\ell,
\qquad 0\le k\le N_r .
\end{equation}
Once $z_k\ge0$ holds for all $k$, the constraint $a_{k+1}\le cz_k+s_k$ follows from
$z_{k+1}\ge0$.

\paragraph{Allocation}
For parameters $\beta\ge0$ and $\gamma\ge0$, we place shifts proportional to the noise
variance on every past step, training included, at a common rate $\beta$, and at rate
$\gamma$ on the current episode:
\begin{equation}
\label{eq:CoSU-profile}
a_{k+1}
=
\begin{cases}
\beta\,(2\eta\sigma_0^2)\,c^{N_r-1-k}, & 0\le k<T,\\[2pt]
\beta\,(2\eta\sigma_j^2)\,c^{N_r-1-k}, & H_j\le k<N_j,\ 1\le j<r,\\[2pt]
\gamma\,(2\eta\sigma_r^2)\,c^{N_r-1-k}, & H_r\le k<N_r .
\end{cases}
\end{equation}
In particular, no shift is placed on a step without noise. Since
\[
\sum_{k=H_j}^{N_j-1}2\eta\sigma_j^2c^{2(N_r-1-k)}=2\eta\sigma_j^2c^{2K(r-j)}S_2(K)
\quad\text{and}\quad
\sum_{k=0}^{T-1}2\eta\sigma_0^2c^{2(N_r-1-k)}=c^{2Kr}W_0,
\]
the propagated variance of
all the past steps is $c^{2K}W_{r-1}=P_r$, that of the steps up to the end of episode $q$
is $c^{2K(r-q)}W_q$, and that of the current episode is $Q_r:=2\eta\sigma_r^2S_2(K)$.
Writing $x:=\beta P_r$ for the amount hidden by the past steps,
the constraint~\eqref{eq:CoSU-linear} and the GDP parameter~\eqref{eq:bok-gdp} become
\begin{equation}
\label{eq:CoSU-two-param}
x+\gamma Q_r=b_r,
\qquad
\mu^2=\frac{x^2}{P_r}+\frac{(b_r-x)^2}{Q_r} .
\end{equation}
For a given $x$, spreading it over the past steps in proportion to their variance
minimizes their cost $\sum a_{k+1}^2/(2\eta\sigma_{k+1}^2)$, by Cauchy--Schwarz, which is
why we use the profile~\eqref{eq:CoSU-profile}.

\begin{lemma}[Admissibility]
\label{lem:admissible}
The allocation~\eqref{eq:CoSU-profile} satisfies~\eqref{eq:CoSU-partial} as soon as
$0\le x\le\bar x_r$ and $\gamma Q_r=b_r-x$ (with $\gamma=0$ when $x=b_r$).
\end{lemma}

\begin{proof}
Let $F(k)$ be the difference between the right-hand side and the left-hand side
of~\eqref{eq:CoSU-partial}, so that $F(0)=0$ and we must show $F(k)\ge0$ for all $k$.
Within a phase of constant variance (the training phase, or an episode $j\le r$), the
increment of $F$ at step $\ell$ is
\[
F(\ell+1)-F(\ell)=c^{N_r-1-\ell}\bigl(s_\ell-\rho\,c^{N_r-1-\ell}\bigr),
\]
with $\rho\ge0$ constant on the phase and $s_\ell$ constant on the phase by
\eqref{eq:CoSU-sk}. Since $c^{N_r-1-\ell}$ increases with $\ell$, the increments change
sign at most once, from positive to negative: $F$ first increases, then decreases, on each
phase, and its minimum over a phase is attained at one of its two ends. It is therefore
enough to check $F\ge0$ at the ends of the phases, i.e., at $k=N_q$ for $0\le q\le r$
(with $N_0=T$ and $F(0)=0$). For $q<r$, by~\eqref{eq:CoSU-available} and the computation
of the propagated variances,
\[
F(N_q)=g_rc^{K(r-q)}I_q-\frac{x}{P_r}\,c^{2K(r-q)}W_q ,
\]
which is nonnegative since $x\le\bar x_r\le P_r\,g_rI_q/(c^{K(r-q)}W_q)
=g_rc^{K(r-q)}I_q\,P_r/(c^{2K(r-q)}W_q)$ when
$W_q>0$, and trivially when $W_q=0$. For $q=r$, $F(N_r)=b_r-x-\gamma Q_r=0$.
\end{proof}

\subsection{Proof of Theorem~\ref{thm:eps-delta-tracking}}

Write $b:=b_r$, $P:=P_r$, $x:=x_r$ and $\bar\mu:=\bar\mu_r$. If $P=0$, then $x=0$, no shift
is placed on the past steps, and $Q_r=b^2/\bar\mu^2$ gives $\mu^2=\bar\mu^2$. Let $P>0$, so
that $W_{r-1}>0$; the term $q=r-1$ of the cap is $P\,g_rI_{r-1}/(c^KW_{r-1})=g_rc^KI_{r-1}=b$, hence
$x\le b$. We check that $x^2/P<\bar\mu^2$ when $x<b$. If $x=\bar\mu^2P/b<b$, then
$\bar\mu^2P<b^2$ and $x^2/P=\bar\mu^4P/b^2<\bar\mu^2$. If $x=\bar x_r<\bar\mu^2P/b$, then
either $\bar\mu^2P\le b^2$ and $x^2/P<\bar\mu^4P/b^2\le\bar\mu^2$, or $\bar\mu^2P>b^2$ and
$x^2/P<b^2/P<\bar\mu^2$. In both cases, the choice
\[
Q_r=\frac{(b-x)^2}{\bar\mu^2-x^2/P}
\]
is well defined, i.e., $\sigma_r^2=Q_r/(2\eta S_2(K))$ is the noise variance of
Theorem~\ref{thm:eps-delta-tracking}, and by Lemma~\ref{lem:admissible}, it gives an admissible allocation
with $\mu^2=x^2/P+(b-x)^2/Q_r=\bar\mu^2$. If $x=b$, then $b\le\bar\mu^2P/b$; we take
$\sigma_r=0$ and $\gamma=0$, and $\mu^2=b^2/P\le\bar\mu^2$.

In all cases, Theorem~\ref{thm:bok-shifted-certificate} gives
\(\mathcal T(\theta_{N_r},\theta_{N_r}^{\langle r\rangle})\ge \mathcal G_{\bar\mu_r}\), and
Lemma~\ref{lem:gdp-to-dp} gives
\(\theta_{N_r}\simeq_{\varepsilon_r,\delta_r}\theta_{N_r}^{\langle r\rangle}\). The argument uses
$\bar\mu_r$ only through this last implication: the theorem holds for any level $\bar\mu_r$ such
that $\bar\mu_r$-GDP implies $(\varepsilon_r,\delta_r)$-indistinguishability, in particular the
exact conversion of Dong et al.~\cite{dong2022gaussian} used in the experiments. Since $r$, the
dataset and the request sequence are arbitrary, CoSU satisfies
counterfactual sequential unlearning at every request. Finally, $x_r$ and $\sigma_r$ are computed from
$(g_r,\bar\mu_r)$, from $(I_q)_{q<r}$, which depend only on $(T,K,\eta,M,c)$, and from
$(W_q)_{q<r}$, which depend only on $\sigma_0$ and the previous noise scales: they are
public. \hfill$\square$

\subsection{Proof of Corollary~\ref{cor:pay-as-you-go}}

Write $J_q:=c^KI_{q-1}=\eta Mc^K(1-c^{H_q})/(1-c)$, so that $b_q=g_qJ_q$ and
$J_q^2/(2\eta S_2(K))=\Gamma(1-c^{H_q})^2/(1-c^{2K})$.

\paragraph{(a)}
Let $P:=P_r>0$ and $F(x):=(b_r-x)^2/(\bar\mu_r^2-x^2/P)$. A direct computation gives
\[
F'(x)=\frac{2(b_r-x)\,(x\,b_r-\bar\mu_r^2P)}{P\,(\bar\mu_r^2-x^2/P)^2},
\]
which is nonpositive on $[0,x_r]$, since there $x\le b_r$ and $x\le\bar\mu_r^2P/b_r$.
Hence $Q_r=F(x_r)\le F(0)=b_r^2/\bar\mu_r^2$ (immediate when $P=0$ or $\sigma_r=0$), i.e.,
$\sigma_r^2\le h_r^2J_r^2/(2\eta S_2(K))=h_r^2\Gamma(1-c^{H_r})^2/(1-c^{2K})\le h_r^2\Gamma/(1-c^{2K})$.

\paragraph{(b)}
Let $\sigma_0=0$, so that $W_0=0$ and the cap only involves $q\ge1$. We prove more generally
that, if effective sizes do not drop too fast,
\begin{equation}
\label{eq:slow-decrease}
h_j\le c^{-K(r-j)/2}\,h_r\qquad\text{for all } 1\le j<r\le R,
\end{equation}
then $x_r=\bar\mu_r^2P_r/b_r$ and $W_r=h_r^2J_r^2$ for $r\le R$. For $r=1$, $P_1=0$ and
$W_1=Q_1=b_1^2/\bar\mu_1^2=h_1^2J_1^2$. Assume the claim for every $q<r$, so that
$P_r=c^{2K}h_{r-1}^2J_{r-1}^2$. For $1\le q<r$, the inequality
$\bar\mu_r^2P_r/b_r\le P_r\,g_rI_q/(c^{K(r-q)}W_q)$ is equivalent to
$c^{K(r-q)}h_q^2J_q^2\le h_r^2J_rI_q$, which follows from~\eqref{eq:slow-decrease},
$J_q\le J_r$ and $J_q\le I_q$. Hence $x_r=\bar\mu_r^2P_r/b_r$ and, by the computation in the
proof of Theorem~\ref{thm:eps-delta-tracking},
\[
Q_r=\frac{b_r^2-\bar\mu_r^2P_r}{\bar\mu_r^2}=h_r^2J_r^2-c^{2K}h_{r-1}^2J_{r-1}^2,
\qquad
W_r=P_r+Q_r=h_r^2J_r^2 .
\]
With constant sizes and privacy levels, \eqref{eq:slow-decrease} holds and $h_r=g/\bar\mu$.
Dividing $Q_1=h^2J_1^2$ by $2\eta S_2(K)$ gives $\sigma_1^2$. For $r\ge2$, with $x:=c^{H_r}$
and $H_{r-1}=H_r-K$,
\[
(1-x)^2-c^{2K}(1-c^{-K}x)^2=(1-c^{2K})-2x(1-c^K)=(1-c^K)(1+c^K-2x),
\]
and dividing by $1-c^{2K}=(1-c^K)(1+c^K)$ gives $\sigma_r^2$.

\subsection{Proof of Proposition~\ref{prop:training-noise}}
For $r=1$, the cap reduces to its term $q=0$, which equals $b_1$, so
$x_1=\min\{\bar\mu_1^2P_1/b_1,\ b_1\}$ with $P_1=c^{2K}W_0=c^{2K}\,2\eta\sigma_0^2S_2(T)$. If
$\bar\mu_1^2P_1\ge b_1^2$, then $x_1=b_1$ and $\sigma_1=0$; otherwise
$Q_1=(b_1^2-\bar\mu_1^2P_1)/\bar\mu_1^2$. Since $b_1=g_1c^KI_0$ and, by
Proposition~\ref{prop:dp-provisioned},
$\sigma_0^\star(g_1,\bar\mu_1)^2=g_1^2I_0^2/(2\eta S_2(T)\bar\mu_1^2)$, we have
$b_1^2/\bar\mu_1^2=c^{2K}\,2\eta S_2(T)\,(\sigma_0^\star)^2$, hence
$Q_1=c^{2K}\,2\eta S_2(T)\,((\sigma_0^\star)^2-\sigma_0^2)_+$, and dividing by
$2\eta S_2(K)$ gives the formula.\hfill$\square$

\section{Proof of Theorem~\ref{thm:final-excess-risk}}
\label{sec:utility-proof}

We next record the contraction estimates used in the utility analysis.

\begin{lemma}[One-step contraction with additive noise]
\label{lem:one-step-contraction}
Let $f$ be $m$-strongly convex and $L$-smooth, with minimizer $\theta^\star$.
Consider the noisy gradient step
\[
\theta^+
=
\theta-\eta\nabla f(\theta)+\sqrt{2\eta}\,\sigma Z,
\qquad
Z\sim\mathcal N(0,I_p),
\]
with $\eta\le 1/L$. Then
\[
\E\|\theta^+-\theta^\star\|^2
\le
c^2\E\|\theta-\theta^\star\|^2
+
2\eta p\sigma^2.
\]
\end{lemma}

\begin{proof}
Let $\Delta:=\theta-\theta^\star$. Since $\nabla f(\theta^\star)=0$,
\[
\theta^+-\theta^\star
=
\Delta-\eta(\nabla f(\theta)-\nabla f(\theta^\star))
+
\sqrt{2\eta}\,\sigma Z.
\]
Conditioning on $\theta$ and using $\E[Z\mid\theta]=0$, we obtain
\[
\E\!\left[\|\theta^+-\theta^\star\|^2\mid \theta\right]
=
\left\|
\Delta-\eta(\nabla f(\theta)-\nabla f(\theta^\star))
\right\|^2
+
2\eta\sigma^2\E\|Z\|^2.
\]
Since $Z\sim\mathcal N(0,I_p)$, we have
\[
\E\|Z\|^2=p.
\]
Moreover, since $f$ is $m$-strongly convex and $L$-smooth, the gradient map
$\theta\mapsto \theta-\eta\nabla f(\theta)$ is contractive for
$\eta\le1/L$, with contraction factor $c=1-\eta m$. Therefore
\[
\left\|
\Delta-\eta(\nabla f(\theta)-\nabla f(\theta^\star))
\right\|
\le
c\|\Delta\|.
\]
Taking expectations gives the result.
\end{proof}

\begin{lemma}[Episode contraction bound]
\label{lem:segment-bound}
Fix an episode of length $K$ on an objective $f$ with minimizer $\theta^\star$,
and assume that the noise scale is constant and equal to $\sigma$ on that episode.
Then
\[
\E\|\theta_K-\theta^\star\|^2
\le
c^{2K}\E\|\theta_0-\theta^\star\|^2
+
2\eta p\sigma^2S_2(K),
\]
where
\[
S_2(K):=\sum_{u=0}^{K-1}c^{2u}
=
\frac{1-c^{2K}}{1-c^2}.
\]
\end{lemma}

\begin{proof}
Applying Lemma~\ref{lem:one-step-contraction} recursively gives
\[
\E\|\theta_{u+1}-\theta^\star\|^2
\le
c^2\E\|\theta_u-\theta^\star\|^2
+
2\eta p\sigma^2.
\]
Unrolling from $u=0$ to $u=K-1$ yields
\[
\E\|\theta_K-\theta^\star\|^2
\le
c^{2K}\E\|\theta_0-\theta^\star\|^2
+
2\eta p\sigma^2\sum_{u=0}^{K-1}c^{2u},
\]
which is the desired bound.
\end{proof}

We now propagate the previous episode bound along the sequential unlearning trajectory.

\begin{lemma}[Sequential utility recursion]
\label{lem:seq-recursion}
Define
\[
\Delta_r
:=
\E\|\theta_{N_r}-\theta_r^\star\|^2.
\]
Let $\alpha>0$ and set
\[
\rho:=(1+\alpha)c^{2K}.
\]
Assume $\rho<1$. Then, for every $r\ge1$,
\[
\Delta_r
\le
\rho \Delta_{r-1}
+
\left(1+\frac1\alpha\right)c^{2K}
\|\theta_{r-1}^\star-\theta_r^\star\|^2
+
2\eta p\sigma_r^2S_2(K).
\]
Consequently,
\[
\Delta_r
\le
\rho^r \Delta_0
+
\left(1+\frac1\alpha\right)c^{2K}
\sum_{s=1}^r
\rho^{r-s}\|\theta_{s-1}^\star-\theta_s^\star\|^2
+
2\eta pS_2(K)
\sum_{s=1}^r
\rho^{r-s}\sigma_s^2.
\]
\end{lemma}

\begin{proof}
Applying Lemma~\ref{lem:segment-bound} to the $r$-th corrective episode gives
\[
\Delta_r
\le
c^{2K}\E\|\theta_{H_r}-\theta_r^\star\|^2
+
2\eta p\sigma_r^2S_2(K).
\]
Since $\theta_{H_r}=\theta_{N_{r-1}}$, we have
\[
\theta_{H_r}-\theta_r^\star
=
(\theta_{N_{r-1}}-\theta_{r-1}^\star)
+
(\theta_{r-1}^\star-\theta_r^\star).
\]
By Young's inequality, for any $\alpha>0$,
\[
\|u+v\|^2
\le
(1+\alpha)\|u\|^2
+
\left(1+\frac1\alpha\right)\|v\|^2.
\]
Therefore,
\[
\E\|\theta_{H_r}-\theta_r^\star\|^2
\le
(1+\alpha)\Delta_{r-1}
+
\left(1+\frac1\alpha\right)
\|\theta_{r-1}^\star-\theta_r^\star\|^2.
\]
Multiplying by $c^{2K}$ and using
$\rho=(1+\alpha)c^{2K}$ gives the one-step recursion. Unrolling it from
$s=1$ to $r$ yields the second bound.
\end{proof}

We now control the displacement between two consecutive minimizers.

\begin{lemma}[Shift of adjacent minimizers]
\label{lem:shift-minimizers}
For every $r\ge1$,
\[
\|\theta_{r-1}^\star-\theta_r^\star\|
\le
\frac{g_rM}{m}.
\]
\end{lemma}

\begin{proof}
Since $f_{r-1}=f_r+\sum_{d\in G_r}\ell(\cdot,d)$, optimality gives
\(
\nabla f_{r-1}(\theta_{r-1}^\star)=0
\) and
\(
\nabla f_r(\theta_r^\star)=0
\).
Thus
\[
\nabla f_r(\theta_{r-1}^\star)-\nabla f_r(\theta_r^\star)
=
-\sum_{d\in G_r}\nabla \ell(\theta_{r-1}^\star,d).
\]
By $m$-strong convexity of $f_r$,
\[
m\|\theta_{r-1}^\star-\theta_r^\star\|
\le
\|\nabla f_r(\theta_{r-1}^\star)-\nabla f_r(\theta_r^\star)\|
\le g_rM,
\]
where the last inequality uses that each $\ell(\cdot,d)$ is $M$-Lipschitz.
\end{proof}

We first state a general excess-risk bound, then specialize it using the
calibration of Theorem~\ref{thm:eps-delta-tracking}.

\begin{lemma}[General excess-risk bound]
\label{lem:seq-risk-abstract}
Let
\[
S_2(K)=\frac{1-c^{2K}}{1-c^2},
\qquad
\rho=(1+\alpha)c^{2K}<1.
\]
Then
\begin{align*}
\E\!\left[f_r(\theta_{N_r})-f_r(\theta_r^\star)\right]
&\le
\frac{L}{2}
\Bigg[
\rho^r \Delta_0
+
\left(1+\frac1\alpha\right)c^{2K}
\sum_{s=1}^r
\rho^{r-s}\|\theta_{s-1}^\star-\theta_s^\star\|^2 \\
&\hspace{2.3cm}
+
2\eta pS_2(K)
\sum_{s=1}^r
\rho^{r-s}\sigma_s^2
\Bigg],
\end{align*}
where the initial learning error satisfies
\[
\Delta_0
\le
c^{2T}\E\|\theta_0-\theta_0^\star\|^2
+
2\eta p\sigma_0^2S_2(T),
\qquad
S_2(T):=\frac{1-c^{2T}}{1-c^2}.
\]
\end{lemma}

\begin{proof}
The bound on $\Delta_0$ follows from Lemma~\ref{lem:segment-bound} applied to the
initial learning phase on $f_0$. By $L$-smoothness of $f_r$,
\[
f_r(\theta_{N_r})-f_r(\theta_r^\star)
\le
\frac{L}{2}\|\theta_{N_r}-\theta_r^\star\|^2.
\]
Taking expectations and applying Lemma~\ref{lem:seq-recursion} gives the result.
\end{proof}

\begin{proof}[Proof of Theorem~\ref{thm:final-excess-risk}]
We apply Lemma~\ref{lem:seq-risk-abstract} with
\[
\alpha:=c^{-K}-1,
\qquad
\rho=(1+\alpha)c^{2K}=c^K<1,
\qquad
1+\frac1\alpha=\frac{1}{1-c^K}.
\]

\paragraph{Initial term}
By the bound on $\Delta_0$ in Lemma~\ref{lem:seq-risk-abstract},
$\frac L2\rho^r\Delta_0\le c^{Kr}\mathcal E_0$.

\paragraph{Drift term}
By Lemma~\ref{lem:shift-minimizers},
\[
\frac L2\Bigl(1+\frac1\alpha\Bigr)c^{2K}\sum_{s=1}^r\rho^{r-s}\|\theta_{s-1}^\star-\theta_s^\star\|^2
\le
B\,\frac{c^{2K}}{1-c^K}\sum_{s=1}^rc^{K(r-s)}g_s^2 .
\]

\paragraph{Noise term}
Write $Q_s:=2\eta\sigma_s^2S_2(K)$, so that the noise contribution in
Lemma~\ref{lem:seq-risk-abstract} is
\[
\frac L2\,2\eta pS_2(K)\sum_{s=1}^r\rho^{r-s}\sigma_s^2
=
\frac{Lp}{2}\sum_{s=1}^rc^{K(r-s)}Q_s .
\]
Since $1-c=\eta m$, the propagated single-record influence satisfies
\[
J_s:=\frac{b_s}{g_s}=\eta Mc^K\frac{1-c^{H_s}}{1-c}\le \frac{\eta Mc^K}{1-c}=\frac{Mc^K}{m}
\qquad\text{for every } s.
\]
By Corollary~\ref{cor:pay-as-you-go}(a),
$Q_s\le b_s^2/\bar\mu_s^2=h_s^2J_s^2\le h_s^2M^2c^{2K}/m^2$, hence
\[
\frac{Lp}{2}\sum_{s=1}^rc^{K(r-s)}Q_s
\le
p\,B\,c^{2K}\sum_{s=1}^rc^{K(r-s)}h_s^2 .
\]
Combining the three terms proves the theorem.
\end{proof}

\section{Training-time DP provisioned for a maximal group size}
\label{app:proof-provisioning}

\begin{proposition}[Training-time DP provisioned for a maximal group size]
\label{prop:dp-provisioned}
Let $\A$ run $T$ steps of noisy gradient descent on $f_0$ with noise $\sigma_0$
and perform no update at deletion time. The pipeline satisfies
Definition~\ref{def:counterfactual_sequential_unlearning} for all requests of size at most
$g_{\max}$ and level at least $\bar\mu$ as soon as
\[
\sigma_0\ \ge\ \sigma_0^\star(g_{\max},\bar\mu)=
\frac{g_{\max}\,I_0}{\bar\mu\sqrt{2\eta S_2(T)}}.
\]
\end{proposition}

\begin{proof}[Proof of Proposition~\ref{prop:dp-provisioned}]
Let $S_1(T):=\sum_{u=0}^{T-1}c^u$, so that $\eta MS_1(T)=I_0$.
Fix $G$ with $|G|=g\le g_{\max}$. Apply
Theorem~\ref{thm:bok-shifted-certificate} to the training trajectories on $D$ and on
$D\setminus G$, with $N=T$, $s_k=\eta Mg$ and $\sigma_{k+1}=\sigma_0$. Take
$a_{k+1}=\lambda c^{T-1-k}$ with $\lambda:=g\eta MS_1(T)/S_2(T)$. Then
$\sum_kc^{T-1-k}a_{k+1}=\lambda S_2(T)=g\eta MS_1(T)=\sum_kc^{T-1-k}s_k$, so $z_T=0$.
Admissibility~\eqref{eq:CoSU-partial} requires, for every $k\le T$,
\[
\lambda\sum_{u=T-k}^{T-1}c^{2u}\le g\eta M\sum_{u=T-k}^{T-1}c^{u},
\quad\text{i.e.,}\quad
\frac{\sum_{u=T-k}^{T-1}c^{2u}}{\sum_{u=T-k}^{T-1}c^{u}}\le\frac{S_2(T)}{S_1(T)} .
\]
The left-hand side is a weighted average of $c^u$ over $u\in[T-k,T-1]$, which is
nondecreasing as $k$ increases, since the added terms have larger $c^u$; it is
therefore maximal for $k=T$, where equality holds. The GDP parameter is
\[
\mu^2=\frac{\lambda^2S_2(T)}{2\eta\sigma_0^2}
=\frac{(g\eta MS_1(T))^2}{2\eta\sigma_0^2S_2(T)},
\]
which is at most $\bar\mu^2$ under the stated condition, since
$g\le g_{\max}$. Lemma~\ref{lem:gdp-to-dp} and Proposition~\ref{prop:no-provisioning}
give the claim. 
\end{proof}

\section{Guarantees of baseline methods}
\label{app:baseline-guarantees}

This appendix gives the exact calibration of the two baselines of
Figure~\ref{fig:baselines}, Descent-to-Delete (D2D) and Langevin Unlearning. Both certify
that the released model is $(\varepsilon,\delta)$-indistinguishable from a model retrained
on $D_r$ (the cumulative-deletion reference of
Definition~\ref{def:naive_sequential_unlearning}), not our counterfactual guarantee; we use
their original theorems to set their noise and their number of corrective steps.

Both theorems are stated for the mean objective, and for requests of one record
(Descent-to-Delete) or of a fixed number of records (Langevin Unlearning). We restate
them for the sum objective $f_r$ and the constants $M$, $m$, $L$ of
Assumption~\ref{ass:contraction}, and for a request that deletes $g_r\ge1$ records;
single-record requests are the case $g_r=1$. As for CoSU, both baselines are calibrated
for the size $g_r$ of each request, not for the largest one. Only the sensitivity changes:
deleting $g$ records changes the gradient of the sum objective by at most $gM$, and moves
its minimizer by at most $gM/m$.

\paragraph{Descent-to-Delete}
We use the perfect-unlearning guarantee, Theorem~28 of Neel et
al.~\cite{neel2020descenttodeletegradientbasedmethodsmachine}, in which no hidden state is
kept: each request starts from the last published model. Training and unlearning are
gradient descent with step size $2/(L+m)$, which contracts the distance to the minimizer by
$\gamma:=(L-m)/(L+m)$ at each step, and the model released after request $r$ is perturbed
by $\mathcal N(0,\sigma_{\mathrm{D2D},r}^2I_p)$. With $\ell_\delta:=2\log(2/\delta)$, the
calibration is
\begin{align*}
\mathcal I
&=\left\lceil\frac{1}{\log(1/\gamma)}\,
\log\frac{\sqrt{2p}}{(1-\gamma)\bigl(\sqrt{\ell_\delta+\varepsilon}-\sqrt{\ell_\delta}\bigr)}
\right\rceil,\\
\sigma_{\mathrm{D2D},r}
&=\frac{8\,g_rM\,\gamma^{\mathcal I}}
{m\,(1-\gamma^{\mathcal I})\bigl(\sqrt{\ell_\delta+3\varepsilon}-\sqrt{\ell_\delta+2\varepsilon}\bigr)},\\
K^{\mathrm{D2D}}_r
&=\left\lceil\mathcal I+\frac{\log\log(4pr/\delta)+\log\max\{1,g_{r-1}/g_r\}}{\log(1/\gamma)}
\right\rceil,
\end{align*}
with the convention $g_{r-1}:=0$ for $r=1$: the model trained before any request is released without noise. Here
$\mathcal I$ is the smallest base horizon allowed by the theorem and $K^{\mathrm{D2D}}_r$
is the number of gradient steps of request $r$. For requests of equal size, in particular
single-record requests, these are the expressions of the theorem written with the
constants of the sum objective.

The theorem uses one noise scale for all requests; we let it follow the size of each
request. Its proof extends as follows. Let $a_r:=4g_rM/m$, which bounds the displacement
$g_rM/m$ of the minimizer caused by request $r$, let $\tau:=\sigma_{\mathrm{D2D},r}/a_r$,
which does not depend on $r$, and let
$E_r:=\frac{\gamma^{\mathcal I}}{1-\gamma^{\mathcal I}}(1+\tau\sqrt{2p})\,a_r$. If the
unpublished model is within $E_{r-1}$ of the minimizer before request $r$, the published
model from which the request starts is, with the probability of the original proof, at
distance at most
\[
\sigma_{\mathrm{D2D},r-1}\sqrt{2p\log(4pr/\delta)}+E_{r-1}+a_r
\;\le\;\rho_r\log(4pr/\delta)\,\bigl(\tau\sqrt{2p}\,a_r+E_r+a_r\bigr),
\quad\rho_r:=\max\{1,g_{r-1}/g_r\},
\]
from the new minimizer. The $K^{\mathrm{D2D}}_r$ gradient steps multiply this distance by a
factor of at most $\gamma^{\mathcal I}/(\rho_r\log(4pr/\delta))$, which gives $E_r$. The rest of the
proof is unchanged: the released model and the released retrained model are Gaussian with
the same covariance and means at distance at most $2E_r$, and $2E_r/\sigma_{\mathrm{D2D},r}$
does not depend on $r$. The term in $\rho_r$ is the cost of a request smaller than the
previous one: the gradient steps must also contract the larger noise of the previous
release.

The theorem also requires the retrained model to be within $2g_rM\gamma^{\mathcal I}/m$ of
the minimizer, that is
\[
T\;\ge\;\mathcal I+\frac{\log\bigl(R_0\,m/(2g_rM)\bigr)}{\log(1/\gamma)}
\qquad\text{for all } r,
\]
where $R_0$ bounds the distance between the initialization and the minimizer. With a zero
initialization, $\frac{\lambda}{2}\|\theta_r^\star\|^2\le f_r(0)\le n\log 10$ gives
$R_0=\sqrt{2\log(10)/\lambda'}$.

With the constants of Appendix~\ref{app:experimental-details} ($\lambda'=0.01$,
$\varepsilon=0.5$, $\delta=1/n$), $\gamma=0.962$ and $\mathcal I=261$, and the condition on
the training length is $T\ge515$ for single records and $T\ge400$ for writers. The $T=200$
training steps selected for CoSU do not satisfy it. For a fair comparison, all the methods
of Figures~\ref{fig:baselines} and~\ref{fig:baselines-writers}, including CoSU, are
therefore trained for $T=10\kappa+100=610$ steps (the cap of
Appendix~\ref{app:experimental-details}), and a retraining is counted as $610$ steps.
For single records, $K^{\mathrm{D2D}}_r$ goes
from $339$ at the first request to $344$ at the fiftieth, and
$\sigma_{\mathrm{D2D},r}=3.1\times10^{-6}$. For writers,
$\sigma_{\mathrm{D2D},r}=3.1\times10^{-6}\,g_r$, at most $3.9\times10^{-4}$, far below the
noise of CoSU, and a request half the size of the previous one costs $18$ more steps.

\paragraph{Langevin Unlearning}
Learning and unlearning are the noisy gradient iteration of CoSU with a constant noise
scale $\sigma_{\mathrm{LU}}$, fixed at training time: request $r$ runs
$K^{\mathrm{LU}}_r$ steps
$\theta_{k+1}=\theta_k-\eta\nabla f_r(\theta_k)+\sqrt{2\eta}\,\sigma_{\mathrm{LU}}Z_{k+1}$
with $\eta=1/L$. The certificate of Chien et
al.~\cite{chien2025langevinunlearningnewperspective} combines three bounds on R\'enyi
divergences of order $\alpha>1$.

\emph{Initial bound} (their Theorem~3.3, strongly convex case). The models trained on two
datasets that differ by $g$ records are at divergence at most
\[
\varepsilon_0^{(g)}(\alpha)=\frac{\alpha\,g^2M^2}{m\,\sigma_{\mathrm{LU}}^2}.
\]
The original bound, $4\alpha S^2M^2/(m\sigma^2n^2)$, is for the mean objective and the
replacement of $S$ records, whose gradient sensitivity is $2SM/n$; the deletion of $g$
records has sensitivity $gM$ under the sum objective. We drop its factor
$1-e^{-m\eta T}\le1$.

\emph{Decay} (their Theorem~3.2, strongly convex case). $K$ steps on the updated objective
multiply the divergence of order $\alpha$ to the retrained model by $e^{-KR/\alpha}$, with
$R=2\sigma_{\mathrm{LU}}^2\eta/C_{\mathrm{LSI}}$, for a log-Sobolev constant such that
$C_{\mathrm{LSI}}>\sigma_{\mathrm{LU}}^2/m$ and
$\eta\le\min\{\tfrac{2}{m}(1-\tfrac{\sigma_{\mathrm{LU}}^2}{mC_{\mathrm{LSI}}}),\tfrac1L\}$.
We take $C_{\mathrm{LSI}}=2\sigma_{\mathrm{LU}}^2/m$, the value used by Chien et al.\ in
their experiments, for which the condition on the step size is
$\eta\le\min\{1/m,1/L\}$ and
\[
R=\eta m .
\]

\emph{Sequential recursion} (their Corollary~3.4). Let $\varepsilon^{(r)}(\alpha)$ bound
the divergence of order $\alpha$ between the model released after request $r$ and a model
retrained on $D_r$. Then
\begin{align*}
\varepsilon^{(1)}(\alpha)
&=e^{-K^{\mathrm{LU}}_1R/\alpha}\,\varepsilon_0^{(g_1)}(\alpha),\\
\varepsilon^{(r)}(\alpha)
&=e^{-K^{\mathrm{LU}}_rR/\alpha}\,\frac{\alpha-1/2}{\alpha-1}
\Bigl(\varepsilon_0^{(g_r)}(2\alpha)+\varepsilon^{(r-1)}(2\alpha)\Bigr),
\qquad r\ge2 .
\end{align*}
The proof of the corollary is a weak triangle inequality between the released model, the
model retrained on $D_{r-1}$ and the model retrained on $D_r$; it applies unchanged to
requests of different sizes, the divergence between the two retrained models being bounded
by $\varepsilon_0^{(g_r)}$.

Request $r$ is certified when
$\min_{\alpha}\bigl\{\varepsilon^{(r)}(\alpha)+\log(1/\delta)/(\alpha-1)\bigr\}\le\varepsilon$,
and $K^{\mathrm{LU}}_r$ is the smallest integer for which this holds given
$K^{\mathrm{LU}}_1,\dots,K^{\mathrm{LU}}_{r-1}$; the minimum is taken over a grid of orders
between $2$ and $5\times10^4$. This is the procedure of Chien et al.\ (their
Algorithms~6--8) with the constants of the sum objective and the size of each request. As
they do, we treat the model obtained after the $610$ training steps as a sample of the
stationary distribution ($c^{610}\approx6\times10^{-6}$).

The recursion evaluates the bound of request $r-1$ at order $2\alpha$, that of request
$r-2$ at order $4\alpha$, and so on, and corrective steps contract little at high orders.
As a result, $K^{\mathrm{LU}}_r$ grows linearly with $r$, by approximately
$\alpha^\star\log 2/(2R)$ steps per request, where
$\alpha^\star\approx1+\log(1/\delta)/\varepsilon$ is the order imposed by the conversion:
about $460$ steps per request with our constants ($474$ measured on single-record
requests), whatever $\sigma_{\mathrm{LU}}$.

\section{Additional experiments}
\label{app:additional-experiments}

All experiments use the setting of Section~\ref{sec:exp-setting}; the notebook reproduces
every figure and table.

\paragraph{Numerical check of the calibration}
For $100$ random request streams mixing group sizes $g_r\in\{1,2,5,20,100\}$ and privacy levels
$\bar\mu_r\in\{0.1,0.3,1,2\}$, with training noise $\sigma_0\in\{0,0.3,3\}$, half of them
with requests grouped at random into blocks of dependent requests
(Appendix~\ref{app:dependent}), we rebuild the shift allocation of the proof of
Theorem~\ref{thm:eps-delta-tracking}, training steps included, and check the constraints
of Theorem~\ref{thm:bok-shifted-certificate} step by step. All allocations are admissible,
the Gaussian-DP parameter equals the target privacy level up to $10^{-12}$, the new noise satisfies
the bound of Corollary~\ref{cor:pay-as-you-go}(a), the constant-size case coincides with
the closed form (b), and the first request is free exactly when $\sigma_0$ reaches the
level of Proposition~\ref{prop:dp-provisioned}, as stated in
Proposition~\ref{prop:training-noise}.

\begin{figure}[t]
    \centering
    \includegraphics[width=0.85\textwidth,trim=0 0 0 162,clip]{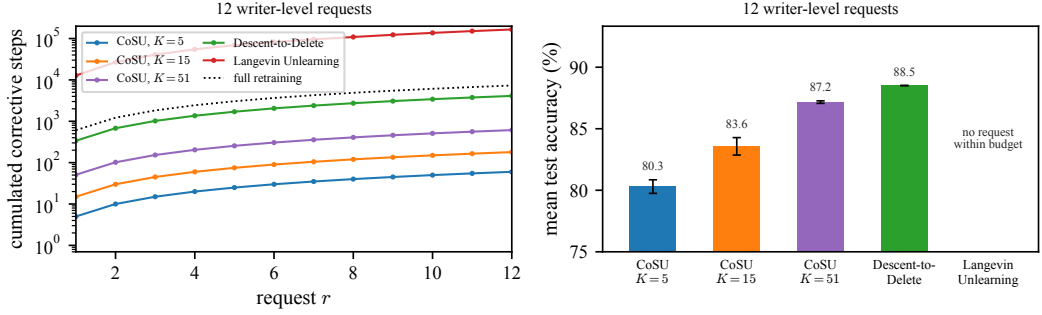}
    \caption{CoSU and certified first-order baselines ($\varepsilon=0.5$) on $12$
    writer-level requests. Left: cumulative number of corrective steps, with the lines of
    Figure~\ref{fig:baselines}. Right: test accuracy averaged over the requests (mean $\pm$
    std over five runs); Langevin Unlearning processes no request within the budget of
    $10$ retrainings.}
    \label{fig:baselines-writers}
\end{figure}

\paragraph{Certified first-order baselines}
We compare CoSU with Descent-to-Delete and Langevin Unlearning
($\sigma_{\mathrm{LU}}=1$), calibrated for the size of each request as in
Appendix~\ref{app:baseline-guarantees} ($\lambda'=0.01$, $\varepsilon=0.5$). They certify
indistinguishability from a model retrained on $D_r$ and not our counterfactual
guarantee, so Figures~\ref{fig:baselines} and~\ref{fig:baselines-writers} compare costs rather than guarantees. For a
fair comparison, all methods, including CoSU, are trained for $T=610$ steps in this
experiment, as required by the certificate of Descent-to-Delete
(Appendix~\ref{app:baseline-guarantees}), and a retraining is counted as $610$ steps. A
method stops before its cumulative number of steps exceeds a budget of $10$ retrainings
($6{,}100$ steps).

We first consider $50$ requests of one record drawn at random. The three methods lose no accuracy on the
requests they process ($88.5\%$ for CoSU, for Descent-to-Delete and for Langevin
Unlearning), and differ in cost. CoSU uses $K=5$ steps per request,
$0.4$ retraining for the whole stream. Descent-to-Delete needs $339$ to $344$ steps per
request, more than half a retraining, and $28.1$ retrainings in total: its base horizon
$\mathcal I=261$ depends on the dimension and on $(\varepsilon,\delta)$, not on the request.
Langevin Unlearning needs one step for each
of the first two requests; its number of steps then grows by $474$ per request,
exceeds a retraining from the third request and reaches $24{,}681$ at the fiftieth,
$1{,}053$ retrainings in total. 
A
larger noise delays this growth without changing its slope: for
$\sigma_{\mathrm{LU}}=10$, the cost exceeds a retraining from request $10$.

On $12$ writer-level requests ($94$ to $129$ records), Descent-to-Delete needs
$339$ to $349$ steps per request and keeps $88.5\%$; CoSU reaches $87.2\%$
with $51$ steps, and $83.6\%$ and $80.3\%$ with $15$ and $5$ steps, operating points
that the horizon of Descent-to-Delete does not allow. Langevin Unlearning needs $12{,}781$
to $14{,}342$ steps per request, more than $20$ retrainings, from the first request: it
processes no request within the budget.

\begin{figure}[t]
    \centering
    \includegraphics[width=0.7\linewidth]{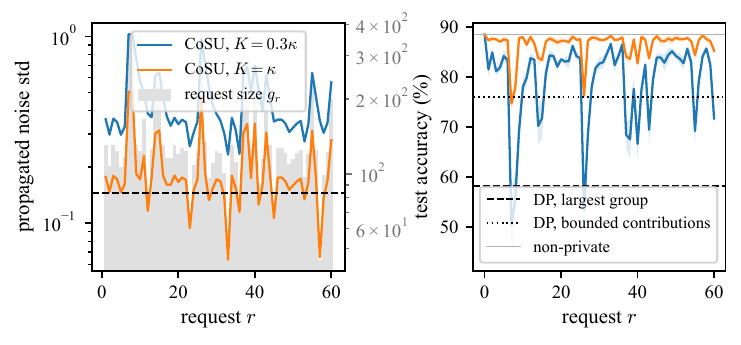}
    \caption{Groups of variable size. Left: request sizes $g_r$ 
    (gray) and propagated noise standard deviation $\sqrt{W_r}$, for Zipf$(3)$
    organizations. Right: test accuracy.}
    \label{fig:femnist-heterogeneity}
\end{figure}

\begin{table}[t]
\centering
\caption{Groups of variable size: accuracy (\%) averaged over the requests (mean $\pm$ std
over five runs) for writers ($100$ requests) and for organizations of Zipf$(a)$ writers who
    leave together ($60$ requests).}
\label{tab:femnist-heterogeneity}
\small
\begin{adjustbox}{max width=\linewidth}
\begin{tabular}{lcccccc}
\toprule
Groups & $\frac{\max}{\text{mean}}$ & $g_{\max}$ & \multicolumn{2}{c}{CoSU} & \multicolumn{2}{c}{DP} \\
 & & & $K=0.3\kappa$ & $K=\kappa$ & largest & bounded \\
\midrule
Writers & 1.2 & 134 & 83.6 $\pm$ 0.3 & \textbf{87.4} $\pm$ 0.1 & 76.6 $\pm$ 1.1 & 79.6 $\pm$ 1.4 \\
Zipf$(4)$ & 3.0 & 375 & 82.4 $\pm$ 0.4 & \textbf{87.2} $\pm$ 0.1 & 58.9 $\pm$ 3.7 & 78.0 $\pm$ 1.7 \\
Zipf$(3)$ & 2.8 & 386 & 79.1 $\pm$ 0.4 & \textbf{86.4} $\pm$ 0.1 & 58.1 $\pm$ 3.8 & 75.9 $\pm$ 2.0 \\
Zipf$(2.5)$ & 2.5 & 390 & 76.7 $\pm$ 0.4 & \textbf{86.0} $\pm$ 0.1 & 57.8 $\pm$ 3.9 & 73.9 $\pm$ 2.4 \\
\bottomrule
\end{tabular}
\end{adjustbox}
\end{table}

\paragraph{Groups of coordinated users}
We group the private writers into organizations whose number of writers follows a Zipf
law of parameter $a$, truncated at $3$ writers. The writers of an organization leave
together: their requests are dependent, so they form a single request, and organizations request independently of each other.
Groups then contain up to $390$ records, about three times the largest writer.
Table~\ref{tab:femnist-heterogeneity} and
Figure~\ref{fig:femnist-heterogeneity} report the results. DP provisioned for the largest
group falls to about $58\%$; contribution bounding, the strongest DP competitor, falls
to $78.0\%$, $75.9\%$ and $73.9\%$. CoSU with $K=\kappa$ keeps $87.2\%$, $86.4\%$ and $86.0\%$,
as its noise follows each request: a large group causes a transient drop of accuracy
(to about $75\%$ at $K=\kappa$, $50\%$ at $0.3\kappa$), whereas the DP learners pay a constant price at every
release. The gap with contribution bounding grows from about $8$ points (writers) to $12$
(Zipf$(2.5)$).

\paragraph{Privacy levels chosen at deletion time}
Each writer-level request chooses its level: $\varepsilon_r\in\{0.25,1,4\}$ with
probabilities $(0.2,0.5,0.3)$, $K=0.3\kappa$, five runs (Table~\ref{tab:femnist-eps}). A DP learner
must be provisioned for the strictest level, $\varepsilon_{\min}=0.25$, and so must an
unlearning procedure that ignores the requested levels. Certifying each request at its own level,
CoSU reaches $81.3\%$, against $72.6\%$ when every request is certified at $\varepsilon_{\min}$:
only the requests that ask for the strictest level pay for it. Both DP learners stay below $70\%$.

\begin{table}[t]
\centering
\caption{Levels chosen at deletion time ($\varepsilon_r\in\{0.25,1,4\}$ with probabilities
$0.2,0.5,0.3$): accuracy (\%) averaged over requests, mean $\pm$ std over five runs.}
\label{tab:femnist-eps}
\small
\begin{adjustbox}{max width=\linewidth}
\begin{tabular}{lc}
\toprule
Method & accuracy \\
\midrule
CoSU, each request at its level $\varepsilon_r$ & \textbf{81.3} $\pm$ 0.4 \\
CoSU, every request at $\varepsilon_{\min}$ & 72.6 $\pm$ 0.3 \\
DP, largest group, $\varepsilon_{\min}$ & 68.4 $\pm$ 1.9 \\
DP, bounded contributions ($g_0=67$), $\varepsilon_{\min}$ & 69.9 $\pm$ 1.1 \\
\bottomrule
\end{tabular}
\end{adjustbox}
\end{table}

\begin{figure}[t]
    \centering
    \includegraphics[width=0.9\textwidth]{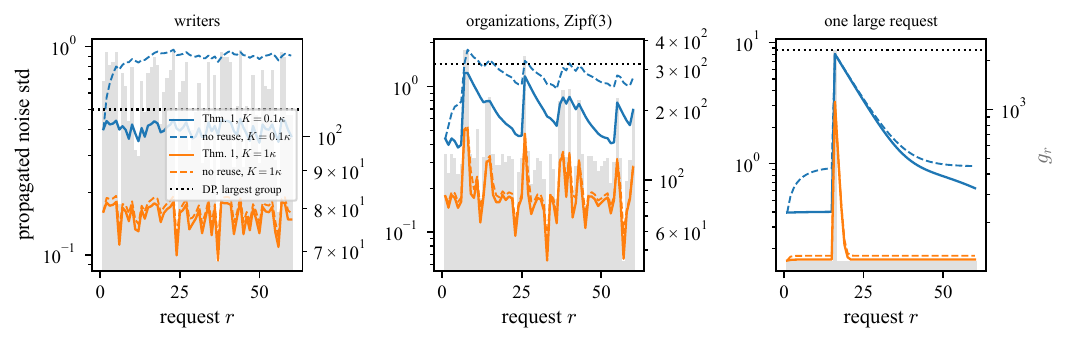}
    \caption{Propagated noise standard deviation $\sqrt{W_r}$ of CoSU (solid) and of its
    variant without reuse of past noise (dashed), with $K=0.1\kappa$ and $K=\kappa$, and of DP
    provisioned for the largest group (dotted), for three size profiles (gray bars: $g_r$).}
    \label{fig:sigma-profiles}
\end{figure}

\paragraph{Noise profiles}
Figure~\ref{fig:sigma-profiles} shows the propagated noise for writers, for Zipf$(3)$
organizations, and for writer-sized requests with a single request twenty times larger.
The noise follows the current request: a large request produces a spike that decays
geometrically and does not raise the noise of the following small requests. When
$K=0.1\kappa$, reusing past noise divides the noise by $2.2$ for writers, $1.8$ for
organizations and $1.2$ for the single large request (median ratios); it matters little
when $K=\kappa$ (median ratio $1.07$).

\begin{figure}[t]
    \centering
    \includegraphics[width=0.6\linewidth]{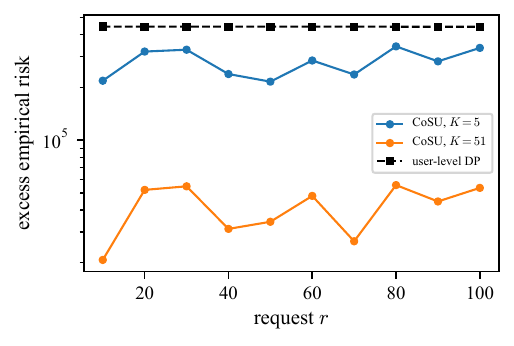}
    \caption{Excess empirical risk $f_r(\theta_r)-f_r(\theta_r^\star)$ along the writer-level
    stream, at the common regularization $\lambda'=0.01$ (five runs).}
    \label{fig:writers-excess}
\end{figure}

\paragraph{Excess risk at a common regularization}
To compare CoSU and DP independently of the selected regularization,
Figure~\ref{fig:writers-excess} reports the excess empirical risk at $\lambda'=0.01$. DP
provisioned for the largest writer stays at $4.4\times10^5$ (for a sum objective over
$279{,}306$ records), CoSU with $K=0.1\kappa$ between $2.2\times10^5$ and $3.4\times10^5$, and
CoSU with $K=\kappa$ between $2.1\times10^4$ and $5.6\times10^4$: the ordering is the same as for
the accuracy.

\begin{figure}[t]
    \centering
    \includegraphics[width=0.6\linewidth]{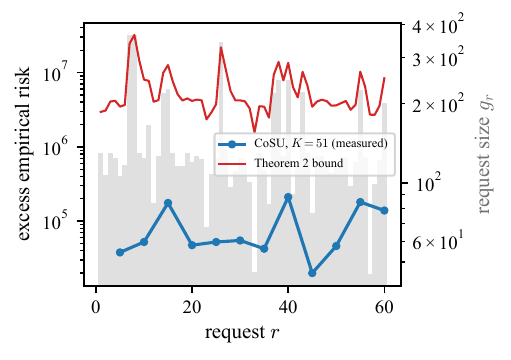}
    \caption{Excess empirical risk along a stream of Zipf$(3)$ organizations ($K=\kappa$,
    five runs), and the bound of Theorem~\ref{thm:final-excess-risk} (gray bars: $g_r$).}
    \label{fig:excess-bound}
\end{figure}

\paragraph{Excess risk and Theorem~\ref{thm:final-excess-risk}}
Along a stream of Zipf$(3)$ organizations (Figure~\ref{fig:excess-bound}), the excess risk of
CoSU follows the recent requests: it rises after the largest
ones (up to $2.1\times10^5$) and returns to its level (median $5.2\times10^4$) within a few requests. The bound
of Theorem~\ref{thm:final-excess-risk} has the same profile and is about $80$ times larger
(median ratio), as expected from a worst-case analysis.

\paragraph{Training noise}
Table~\ref{tab:sigma0} varies the training noise $\sigma_0$, as a fraction of the level
$\sigma_0^\star$ that makes training differentially private for the largest writer
(Proposition~\ref{prop:dp-provisioned}), with $K=\kappa$. The calibration accounts for the training noise: at
$\sigma_0=\sigma_0^\star$, the first request is free, as predicted by
Proposition~\ref{prop:training-noise}. Later requests do not benefit, because the training
noise has contracted by the time they arrive and the previous episodes were run with their
records: the accuracy averaged over the requests is unchanged ($87.3$--$87.4\%$), while the
accuracy before any request drops from $88.4\%$ to $78.8\%$. With
$K=\kappa$, training noise therefore brings nothing, and we use $\sigma_0=0$ throughout.

\begin{table}[t]
\centering
\caption{Training noise $\sigma_0$ (writer-level requests, $K=\kappa$, five runs): accuracy (\%)
before any request and averaged over the requests, and noise of the first and of the later
requests.}
\label{tab:sigma0}
{%
\begin{tabular}{ccccc}
\toprule
$\sigma_0/\sigma_0^\star$ & before & mean over requests & $\sigma_1$ & mean $\sigma_r$, $r\ge2$ \\
\midrule
0 & 88.4 & 87.4 $\pm$ 0.1 & 9.50 & 8.16 \\
0.5 & 85.7 & 87.3 $\pm$ 0.1 & 8.00 & 8.16 \\
1 & 78.8 & 87.4 $\pm$ 0.1 & 0 & 8.16 \\
\bottomrule
\end{tabular}
}
\end{table}

\section{Experimental details}
\label{app:experimental-details}

\paragraph{Data}
We use the digits (classes $0$--$9$) of FEMNIST~\cite{caldas2018leaf}, in the version
distributed with writer identifiers (\texttt{flwrlabs/femnist}). Writers are split
uniformly at random (seed $0$) into $500$ public writers ($56{,}384$ images), $300$ test
writers ($33{,}366$ images), $300$ validation writers and $2{,}480$ private writers
($279{,}306$ images, $113$ records per writer on average, at most $134$). The public writers
are never deleted and are only used to fit the feature map; the validation writers are
only used for hyperparameter selection.

\paragraph{Feature map}
We use the ResNet-18 of torchvision pretrained on ImageNet (\texttt{IMAGENET1K\_V1}), frozen.
The $28\times28$ gray-level images are resized to $64\times64$ (bilinear), replicated on the
three channels and normalized with the ImageNet statistics; the features are the spatial
average of the output of the second residual stage ($128$ dimensions). Each coordinate is
centered and divided by its standard
deviation on the public writers, and every feature vector is normalized to unit $\ell_2$
norm. No FEMNIST image is used to train the network, and the standardization uses public
writers only, so the feature map does not depend on the private data. The non-private
model reaches $88.4\%$ test accuracy. 

\paragraph{Model and constants}
The certified model is a multiclass softmax regression with parameter
$\theta\in\mathbb R^{128\times10}$ ($p=1280$) and full-batch sum objective
\[
f_r(\theta)=\sum_{i\in D_r}\ell(\theta;x_i,y_i)+\frac{\lambda}{2}\|\theta\|_F^2,
\qquad \lambda=\lambda'n,
\]
with $n$ the (public) size of the initial training set. The ridge penalty gives
$m=\lambda$. For a single example,
$\nabla^2\ell=(x_ix_i^\top)\otimes(\operatorname{diag}(\pi_i)-\pi_i\pi_i^\top)$, where
$\pi_i$ is the vector of predicted probabilities, with
$\operatorname{diag}(\pi_i)-\pi_i\pi_i^\top\preceq\frac12I$; since $\|x_i\|_2=1$ and
$|D_r|\le n$, every retained objective is $L$-smooth with $L=n/2+\lambda$. This bound
only depends on public quantities, unlike the tighter $\lambda_{\max}(X^\top X)/2+\lambda$,
which would make the calibration depend on the private records. Since
$\|\nabla_\theta\ell\|_F=\|x_i\|_2\|\pi_i-e_{y_i}\|_2\le\sqrt2$, we take $M=\sqrt2$.
This gives $\kappa=1/(2\lambda')+1$ ($\kappa=51$ for the selected $\lambda'=0.01$) and
$c=1-1/\kappa$. The number of training steps $T$, which is also the cost of a retraining,
is measured for each $\lambda'$ on the validation writers: it is the smallest number of
steps, on a fixed grid of checkpoints, whose validation accuracy without noise is within
$0.1$ point of the best accuracy reached within $10\kappa+100$ steps ($T=200$ for
$\lambda'=0.01$, about $4\kappa$). This cap is our choice: gradient descent with step size
$1/L$ contracts the distance to the minimizer by $c=1-1/\kappa$ at each
step~\cite{nesterov2004introductory}, hence by about $e^{-10}$ after $10\kappa$ steps, and
the additional $100$ steps cover small values of $\kappa$. Training is deterministic
($\sigma_0=0$) unless stated otherwise, and CoSU uses a number $K$ of corrective steps per
request proportional to $\kappa$, rounded to an integer ($K=5$, $15$, $31$ and $51$ for
$0.1\kappa$, $0.3\kappa$, $0.6\kappa$ and $\kappa$).

\paragraph{Privacy parameters}
Requests are certified at $(\varepsilon,\delta)=(0.5,1/n)$ unless stated otherwise. The Gaussian-DP level is the largest $\mu$ such that
$\mu$-GDP implies $(\varepsilon,\delta)$-DP, obtained by inverting the exact
conversion of Dong et al.~\cite{dong2022gaussian} ($\mu=0.133$ for $\varepsilon=0.5$ and
$\mu=0.253$ for $\varepsilon=1$, with $\delta=1/n$).
DP learners use the same conversion.

\paragraph{Competitors}
By Proposition~\ref{prop:no-provisioning}, training-time DP satisfies our definition for
the requests it is provisioned for. We compare with two DP learners, trained by
noisy gradient descent calibrated by Proposition~\ref{prop:dp-provisioned} (the same
shifted-interpolation analysis as CoSU) for the number of steps $T$ of their own
$\lambda'$, and performing no update at deletion time:
\emph{DP, largest group}, provisioned for the largest
user of the training set, and \emph{DP, bounded contributions}
\cite{amin2019bounding,epasto2020smoothly}, which keeps at most $g_0$ records per user and
is provisioned for $g_0$. We also compare with Langevin Unlearning and Descent-to-Delete 
in Figure~\ref{fig:baselines} and Appendix~\ref{app:additional-experiments}.
Results are averaged over five runs (mean $\pm$ standard
deviation) unless stated otherwise.

\paragraph{Streams and users}
Requests are drawn uniformly at random without replacement among users (seed $1$),
independently of the data and of the models, so that requests are independent
(Assumption~\ref{ass:independent-requests}).
Organizations are formed by drawing, for each organization, a number of writers
from a Zipf distribution of parameter $a$, truncated at $3$, and assigning writers
in a random order. Contribution bounding keeps, for each user, a uniformly random
subset of at most $g_0$ records.

\paragraph{Hyperparameter selection}
All hyperparameters are selected on the validation writers, disjoint from the training and
test writers. The grid is $\lambda'\in\{0.001,0.003,0.01,0.03,0.1,0.3\}$, with the training
length $T$ of each value ($1{,}000$, $500$, $200$, $75$, $25$ and $7$ steps). For CoSU, one
$\lambda'$ is selected once and used in all experiments: CoSU is run at $\varepsilon=0.5$
on a separate stream of $30$ writer-level requests with $50$ steps per request, close to
the largest budget we use ($K=\kappa=51$), and scored by its validation accuracy averaged over the requests and
over three noise draws. This selects $\lambda'=0.01$ ($88.1\%$, against $84.7\%$ for
$\lambda'=0.03$ and $81.3\%$ for $\lambda'=0.003$). The DP learners are trained on the
private writers and evaluated on the validation writers (three noise draws): for every
experiment, they select their own $\lambda'$ in the same grid extended with $\lambda'=1$,
hence their own $T$, and contribution bounding selects its bound
$g_0\in\{8,17,34,67\}$ (from one sixteenth to one half of the largest writer), extended to
$\{134,268\}$ for organizations. The selected values are given in
Table~\ref{tab:tuning}.

The comparison with the DP learners depends on the privacy level. With their
hyperparameters selected again at each level, and $\lambda'=0.01$ kept for CoSU, the best
DP learner reaches $69.9\%$, $79.6\%$ and $87.5\%$ at
$\varepsilon=0.25$, $0.5$, $1$, against $84.3\%$, $87.4\%$ and $88.1\%$ for CoSU with 
$K=\kappa$ (Figure~\ref{fig:tradeoff}, middle). CoSU is $14.4$
and $7.8$ points ahead at $\varepsilon=0.25$ and $0.5$, and on par at $\varepsilon=1$.

\begin{table}[t]
\centering
\caption{Hyperparameters selected on validation writers: ridge parameter $\lambda'$, number of training
steps $T$ and, for DP with bounded contributions, bound $g_0$.}
\label{tab:tuning}
\begin{tabular}{lccc}
\toprule
Groups & CoSU $(\lambda',T)$ & DP largest group $(\lambda',T)$ & DP bounded $(\lambda',T,g_0)$ \\
\midrule
Writers & (0.01, 200) & (0.03, 75) & (0.003, 500, 34) \\
Zipf$(4)$ & (0.01, 200) & (0.1, 25) & (0.003, 500, 34) \\
Zipf$(3)$ & (0.01, 200) & (0.1, 25) & (0.003, 500, 34) \\
Zipf$(2.5)$ & (0.01, 200) & (0.1, 25) & (0.003, 500, 34) \\
Writers, $\varepsilon_r$ at deletion time & (0.01, 200) & (0.1, 25) & (0.03, 75, 67) \\
\bottomrule
\end{tabular}
\end{table}

\paragraph{Membership inference}
We audit the certificate of a target user in the strongest threat model: the attacker
knows all the other records, the algorithms and the sizes and privacy levels of the requests, and
must decide from the final release whether the target's records were in the training set
and deleted (world IN) or never there (world OUT, where the same requests are issued but
delete nothing). To afford many runs, the training set is a sub-population of $200$ private
writers drawn at random ($n=22{,}460$), and all constants are computed for this dataset,
with $\lambda'=0.01$ ($\kappa=51$, $T=200$), $K=0.1\kappa$ and $\varepsilon=1$ ($\bar\mu=0.296$).
The target is deleted either as \emph{one} request, the correct treatment of a user, or
split into $m\in\{2,4,8\}$ requests issued consecutively and each calibrated as if it were
independent: the requests of the same user are dependent, and this treatment violates
Assumption~\ref{ass:independent-requests}. We run the pipeline $1{,}024$ times in each
world, in batches of $64$ runs with independent noise.

An attack assigns a score to each
release and compares it with a threshold. The \emph{white-box} attack uses the parameters
of the released model. It first runs the pipeline without noise in each world, which gives
two models $\theta^{\mathrm{IN}}$ and $\theta^{\mathrm{OUT}}$, and scores a release
$\theta$ by its inner product with their difference,
$\langle\theta,\theta^{\mathrm{IN}}-\theta^{\mathrm{OUT}}\rangle$. For a Gaussian release
centered at $\theta^{\mathrm{IN}}$ or $\theta^{\mathrm{OUT}}$, this score is the
likelihood-ratio statistic. The \emph{black-box} attack only evaluates the released model: its
score is the average loss on the target's records. We report $\hat\mu$, the gap between the mean scores of the two
worlds divided by their pooled standard deviation: for Gaussian scores, this is the
Gaussian-DP level of the attack \cite{dong2022gaussian}, which cannot exceed $\bar\mu$ if
the target is protected; the $95\%$ intervals are percentile bootstrap intervals over the
runs. Besides a natural writer ($101$ records), we use a \emph{canary} user, as in privacy
auditing: $16$ records of a writer are replaced by copies of one of its images with the
next label, so that their gradients stay close to the bound $M$ and are aligned, which is
close to the worst case of the certificate.
Table~\ref{tab:mia} reports both attacks.

\begin{table}[t]
\centering
\caption{Membership inference ($1{,}024$ runs per world, $\bar\mu=0.296$): empirical
Gaussian-DP level $\hat\mu$ of the white-box attack with its $95\%$ bootstrap interval, and
AUC of the white-box and black-box attacks, for a user of $g$ records deleted in one
request ($m=1$) or split into $m$ requests certified separately.}
\label{tab:mia}
\begin{tabular}{lcccccc}
\toprule
User & $g$ & $m$ & $\hat\mu$ & $95\%$ interval & AUC, white-box & AUC, black-box \\
\midrule
canary & 16 & 1 & 0.147 & $[0.061,0.238]$ & 0.544 & 0.542 \\
canary & 16 & 2 & 0.365 & $[0.279,0.458]$ & 0.602 & 0.601 \\
canary & 16 & 4 & 0.565 & $[0.475,0.654]$ & 0.655 & 0.658 \\
canary & 16 & 8 & 0.906 & $[0.814,0.998]$ & 0.742 & 0.751 \\
\midrule
natural & 101 & 1 & -0.018 & $[-0.098,0.082]$ & 0.494 & 0.516 \\
natural & 101 & 2 & 0.009 & $[-0.071,0.099]$ & 0.502 & 0.511 \\
natural & 101 & 4 & 0.017 & $[-0.067,0.095]$ & 0.502 & 0.504 \\
natural & 101 & 8 & 0.138 & $[0.048,0.220]$ & 0.539 & 0.534 \\
\bottomrule
\end{tabular}
\end{table}

\paragraph{Compute}
All experiments run on the GPU of a MacBook Pro, through the MPS backend of PyTorch. The
runs of a point, and the configurations of a grid, are computed in one batch: a
full-batch gradient step on the $279{,}306$ records takes $23$ ms for one run and $11$ ms
for a batch of five runs. The ImageNet features are computed once and cached. The
membership inference
experiment takes less than one minute, the writer-level experiment (including the tuning of the
DP learners) $5$ minutes, the trade-off in $\varepsilon$ $10$ minutes, the groups of
coordinated users $6$ minutes, the levels chosen at deletion time $2$ minutes, and the
appendix experiments $4$ minutes: about half an hour in total.

\end{document}